\documentclass[journal,twoside,web]{ieeecolor}
\usepackage{tmi}
\usepackage{cite}
\usepackage{amsmath,amssymb,amsfonts}
\usepackage{algorithmic}
\usepackage{graphicx}
\usepackage{textcomp}
\usepackage{booktabs}

\usepackage{algorithm}
\usepackage{array}
\usepackage[caption=false,font=normalsize,labelfont=sf,textfont=sf]{subfig}
\usepackage{stfloats}
\usepackage{url}
\usepackage{verbatim}
\usepackage{multirow}
\newcommand{\etal}{\textit{et~al.}}
\newtheorem{Theorem}{Theorem}

\def\BibTeX{{\rm B\kern-.05em{\sc i\kern-.025em b}\kern-.08em
    T\kern-.1667em\lower.7ex\hbox{E}\kern-.125emX}}
\begin{document}
\title{DuMeta++: Spatiotemporal Dual Meta-Learning for Generalizable Few-Shot Brain Tissue Segmentation Across Diverse Ages}
\author{Yongheng~Sun,
        Jun~Shu,
        Jianhua~Ma,
        and~Fan~Wang
\thanks{This work was supported in part by Brain Science and Brain-like Intelligence Technology -- National Science and Technology Major Project (No.~2022ZD0209000), Fundamental and Interdisciplinary Disciplines Breakthrough Plan of the Ministry of Education of China, (No.~JYB2025XDXM101), National Natural Science Foundation of China (No.~T2522028), and Natural Science Basic Research Program of Shaanxi, China (No.~2024JC-TBZC-09) (Corresponding authors: F. Wang and J. Ma)}
\thanks{Y. Sun, and J. Shu, are with the School of Mathematics and Statistics, Xi'an Jiaotong University, Xi'an 710049, China.}
\thanks{F. Wang and J. Ma are with Key Laboratory of Biomedical Information Engineering of Ministry of Education, School of Life Science and Technology, Xi'an Jiaotong University, Xi'an 710049, China. \\E-mail: jhma@xjtu.edu.cn; fan.wang@xjtu.edu.cn}
%\thanks{Y. Sun, J. Ma, and F. Wang are also with the Research Center for Intelligent Medical Equipment and Devices (IMED), Xi'an Jiaotong University, Xi'an 710049, China.}
}

\maketitle

\begin{abstract}
Accurate segmentation of brain tissues from MRI scans is critical for neuroscience and clinical applications, but achieving consistent performance across the human lifespan remains challenging due to dynamic, age-related changes in brain appearance and morphology. While prior work has sought to mitigate these shifts by using self-supervised regularization with paired longitudinal data, such data are often unavailable in practice. To address this, we propose \emph{DuMeta++}, a dual meta-learning framework that operates without paired longitudinal data. Our approach integrates: (1) meta-feature learning to extract age-agnostic semantic representations of spatiotemporally evolving brain structures, and (2) meta-initialization learning to enable data-efficient adaptation of the segmentation model. Furthermore, we propose a memory-bank-based class-aware regularization strategy to enforce longitudinal consistency without explicit longitudinal supervision. We theoretically prove the convergence of our DuMeta++, ensuring stability.  Experiments on diverse datasets (iSeg-2019, IBIS, OASIS, ADNI) under few-shot settings demonstrate that DuMeta++ outperforms existing methods in cross-age generalization. Code will be available at \url{https://github.com/ladderlab-xjtu/DuMeta++}.
\end{abstract}


\section{Introduction}

Accurately segmenting brain tissues from magnetic resonance imaging (MRI) scans is crucial for a range of neuroscience and clinical activities, such as population-level examinations of cortical organization and personalized assessments of neurological conditions~\cite{fischl2012freesurfer}. Yet, a major challenge stems from the highly folded and spatiotemporally evolving brain morphologies across different ages~\cite{9339962}. As depicted in Fig.~\ref{fig:data_example}(a), infants younger than six months can exhibit inverted gray matter (GM) and white matter (WM) intensity contrasts, particularly around the isointense stage, where inter-tissue contrast is extremely low. In contrast, Fig.~\ref{fig:data_example}(b) illustrates how elderly individuals and patients with Alzheimer’s disease often present enlarged cerebrospinal fluid (CSF) compartments and GM atrophy. These pronounced differences compromise both the generalizability and longitudinal consistency of existing learning-based segmentation methods~\cite{milletari2016v, WU2022103541, lee2023fine, fan2022attention}, particularly when only a limited number of annotated samples are available at specific time points.

A number of methods have been proposed to address these domain shifts by learning longitudinally generalizable representations via self-supervised strategies~\cite{ren2022local, ouyang2021self}. Typically, such approaches employ contrastive regularization in a pre-training phase, where paired MRI scans across ages help reinforce consistency. Nonetheless, these strategies face notable practical hurdles: acquiring paired longitudinal data is costly and sometimes infeasible; when the final task is segmentation on an age group unseen during pre-training, a secondary fine-tuning step must be performed, which itself demands sufficient labeled data. Consequently, models trained in this way often struggle to balance generalization to various age domains with the limited supervision available.

\begin{figure}[t]
\setlength{\abovecaptionskip}{-1pt}
\setlength{\belowcaptionskip}{-15pt}
  \centering
\includegraphics[width=0.65\linewidth]{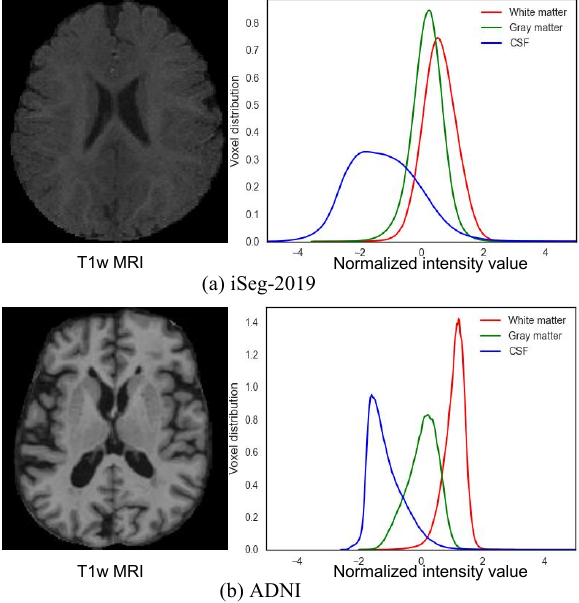}
   \caption{ (a) Brain morphology and tissue contrast of 6-month-old infants from the iSeg-2019 dataset. (b) Aging and Alzheimer's processes of the elderly from
the ADNI dataset.}
\label{fig:data_example}
\vspace{-10pt}
\end{figure}

In this work, we revisit longitudinally consistent representation learning and propose an alternative approach that eliminates the need for paired scans. Our framework is grounded in two key insights: \textbf{(1)} While MRI appearance varies substantially across ages, the semantic distinctions between tissue classes remain relatively stable over time; and \textbf{(2)} Accurate tissue segmentation requires a robust mapping function that integrates global semantic cues with age-specific high-resolution details of the brain.
Building on these insights, we introduce \emph{DuMeta++}, a unified meta-learning framework that jointly optimizes: (i) a time-invariant encoder that extracts semantically consistent features across spatiotemporally dynamic age groups, and (ii) a meta-initialized segmentation head capable of few-shot adaptation to new domains even with just a single labeled MRI scan. 

This work significantly extends our prior conference publication~\cite{sun2023dual} through four key advancements:
a memory-bank-based class-aware regularization mechanism enhancing longitudinal consistency, formal convergence guarantees for the dual meta-learning pipeline, expanded ablation studies validating design choices, and comprehensive benchmarking against state-of-the-art methods across diverse datasets.

%The remainder of this paper is organized as follows. Section II reviews related work on brain tissue segmentation, meta-learning, and longitudinally-consistent representation learning. Section III details the proposed DuMeta++ framework, including the dual meta-learning paradigm and the class-aware regularization strategy, and provides theoretical convergence guarantees. Section IV describes the experimental setup, presents comparative and ablation studies, and discusses the results. Finally, Section V offers a discussion of the findings and concludes the paper.

The main contributions of this paper are four-fold: 
\begin{itemize} 
\item We develop a spatiotemporal dual meta-learning (DuMeta++) framework that harnesses meta-feature learning and meta-initialization learning within a unified bilevel optimization, jointly learning a robust, age-agnostic feature extractor and a segmentation head that can be adapted with minimal labeled data. 
\item We introduce memory-bank-based regularization that enforces class-aware temporal alignment and spatial separation, ensuring tissue-specific feature consistency along ages without requiring paired longitudinal data. 
\item Our meta-learning pipeline only demands cross-sectional training sets and as few as one annotated sample from a new age domain to achieve accurate segmentation, making it practical for diverse lifespan imaging scenarios. 
\item Extensive experiments under the few-shot segmentation setting on datasets from diverse age groups (infant and aging: iSeg-2019, IBIS, OASIS, ADNI) confirm that our framework substantially surpasses existing methods designed for longitudinal consistency. \end{itemize}

\section{Related Work} 
%In this section, we briefly summarize prior research on brain tissue segmentation and few-shot cross-domain segmentation. The latter is organized into three themes: meta-feature learning, meta-initialization learning, and methods for ensuring longitudinally consistent representation learning.

\subsection{Brain Tissue Segmentation} %, SPM~\cite{friston2003statistical}, and BrainSuite~\cite{shattuck2002brainsuite}
Brain tissue segmentation of cerebrospinal fluid (CSF), gray matter (GM), and white matter (WM) from MRI is foundational for quantitative neuroimaging and downstream tasks such as disease monitoring~\cite{ashburner2005unified}. 
Classical pipelines, exemplified by widely used toolkits like FSL~\cite{jenkinson2012fsl} and FreeSurfer~\cite{fischl2012freesurfer}, span region-based, thresholding, clustering, and feature/classification paradigms, often paired with bias-field correction; while these approaches established early baselines, they are sensitive to acquisition artifacts and heterogeneous tissue appearance. 
In contrast, deep learning has become the de facto standard: U-Net–style 2D~\cite{ronneberger2015u} and 3D~\cite{milletari2016v} networks and their variants (with attention~\cite{oktay2018attention}, dense connections~\cite{huang2020unet}, and multi-modal fusion~\cite{zhuang2021aprnet}) report strong Dice scores across adult cohorts, and self-configuring frameworks such as nnU-Net~\cite{isensee2021nnu} further simplify model design. 
Community benchmarks (e.g., MRBrainS~\cite{mendrik2015mrbrains} for adults; iSeg-2019~\cite{sun2021multi} for 6-month-old infants) have standardized evaluations and accelerated progress, and specialized infant pipelines such as iBEAT~\cite{wang2023ibeat} complement these developments.
%iSeg-2017~\cite{wang2019benchmark} and zhou2018unet++, 

Despite strong in-distribution performance, brain tissue contrast and morphology shift across the lifespan, most prominently in infants (isointense phases) and older adults (atrophy and contrast changes), undermining generalization. DuMeta++ addresses this by meta-learning an age-invariant feature extractor and coupling it with a lightweight, rapidly adaptable segmentation head, with class-aware alignment to stabilize CSF/GM/WM representations across age groups.

\subsection{Meta-Feature Learning} Meta-learning generally comprises two core stages: meta-training and meta-testing~\cite{hospedales2021meta}. A prominent category within this domain is meta-feature learning (MFL)~\cite{liu2021investigating}, which seeks to train a shared feature extractor during the meta-training phase to handle multiple, potentially related, tasks. A separate task-specific head is then learned in the meta-testing phase for unseen downstream tasks. Conceptually, the shared feature extractor is treated as an explicit meta-learner, which is commonly optimized through a bilevel scheme.
Two principal gradient computation strategies characterize MFL. In the first, the meta-learner updates its parameters by performing multiple gradient-descent steps on the loss function, implementing an implicit relationship between inner and outer loops~\cite{calatroni2017bilevel,lorraine2018stochastic}. The second approach employs implicit function theory to calculate hyper-gradients~\cite{franceschi2018bilevel, franceschi2017bridge}. The former method is widely adopted in automated differentiation approaches to solve bilevel optimization problems.
For instance, Franceschi \etal~\cite{franceschi2018bilevel} treated the final layer of a classification model as the base-learner while regarding the remaining layers as the meta-learner, updating the latter with back-propagation in a multi-step procedure. 
%~\cite{foo2007efficient, chapelle2002choosing, seeger2008cross, calatroni2017bilevel, lorraine2018stochastic} foo2007efficient, domke2012generic, 
%~\cite{domke2012generic, franceschi2018bilevel, okuno2018ell_p, franceschi2017bridge, baydin2014automatic}

In our work, we decompose a segmentation network into a plug-and-play feature extractor (i.e., the meta-learner) and a task-specific segmentation head (i.e., the base-learner). We carry out meta-training for the feature extractor under an implicit gradient framework to learn anatomical representations that remain consistent over time, leveraging additional class-aware constraints to promote this consistency.

\subsection{Meta-Initialization Learning} A second key branch of meta-learning is meta-initialization learning (MIL), which focuses on discovering favorable initial network parameters that span multiple tasks~\cite{liu2021investigating}. Unlike most MFL methods, MIL typically lacks an explicit meta-learner; its initialization parameters can be viewed as an implicit meta-learner whose relationship to the base-learner emerges at the initialization. Many MIL techniques update this initialization through implicit gradients in the outer loop of a bilevel structure, using relationships established in the inner loop.
Among the early works in this domain, MAML~\cite{finn2017model} utilized the gradient from an initialization step to guide the optimization of the meta-learner. Subsequent efforts investigated how to reduce the computational burden of dealing with Hessian calculations in implicit gradient descent, such as by discarding second-order derivatives~\cite{li2017meta, song2019maml}. ANIL~\cite{raghu2019rapid}, for instance, removed the inner-loop updates for all but the final task-specific head, whereas Reptile~\cite{nichol2018first} proposed first-order algorithms to circumvent second-order overhead. Beyond these algorithmic refinements, several studies explored integrating MIL with domain generalization~\cite{liu2020shape}. For example, SAML~\cite{liu2020shape} simulated domain shifts through pseudo-meta-train and meta-test splits during training.
%~\cite{li2017meta, nichol2018first, song2019maml}~\cite{liu2020shape, liu2021semi, lee2021meta}, lee2021meta

While many existing MIL approaches attempt to initialize an entire network, our DuMeta++ framework streamlines this process by focusing only on a lightweight segmentation head once the feature extractor has already been trained for age-invariant semantic representation. By doing so, we simulate domain shifts during fine-tuning with MIL, reducing computational costs and facilitating rapid adaptation to new domains.

\subsection{Longitudinally-Consistent Representation Learning} One widely studied paradigm for deriving longitudinally-consistent representations is self-supervised learning (SSL), where models are typically pre-trained on data gathered at different time points~\cite{ren2022local,ouyang2021self,ouyang2023lsor,chen2023brain}. Most SSL approaches incorporate contrastive losses to align representations of similar inputs (e.g., patches from the same brain image under different augmentations) and separate dissimilar ones, thereby fostering invariance to temporal shifts.
For example, Ren \etal~\cite{ren2022local} introduced a patch-level spatiotemporal similarity constraint designed to enhance longitudinal consistency, together with additional orthogonality, variance, and covariance terms to avert mode collapse. Another representative work, LNE~\cite{ouyang2021self}, employed a longitudinal neighborhood embedding to promote coherent, low-dimensional trajectories in latent space. Likewise, LSOR~\cite{ouyang2023lsor} utilized soft self-organizing maps and trajectory regularization for coherent aging-related representations. Beyond self-supervised learning, Chen \etal~\cite{chen2023brain} formulated a joint feature regularization strategy to encourage stable segmentations over the lifespan.
%~\cite{chen2019self,chaitanya2020contrastive, zeng2021positional, park2020contrastive}
%building on SimSiam~\cite{chen2021exploring}, 

Our DuMeta++ strategy advances these ideas by jointly training a feature extractor and a segmentation head, aiming for both longitudinal consistency and efficient adaptation. In contrast to traditional SSLs, we do not demand paired longitudinal scans for meta-training; rather, our plug-and-play encoder is guided by class-aware regularization to preserve semantic distinctions over time. Furthermore, only the final segmentation head requires fine-tuning for a novel age group, facilitating streamlined adaptation with minimal labeled data.

\section{Method} \label{sec:method}

\begin{figure*}[t] \centering \includegraphics[width=0.75\linewidth]{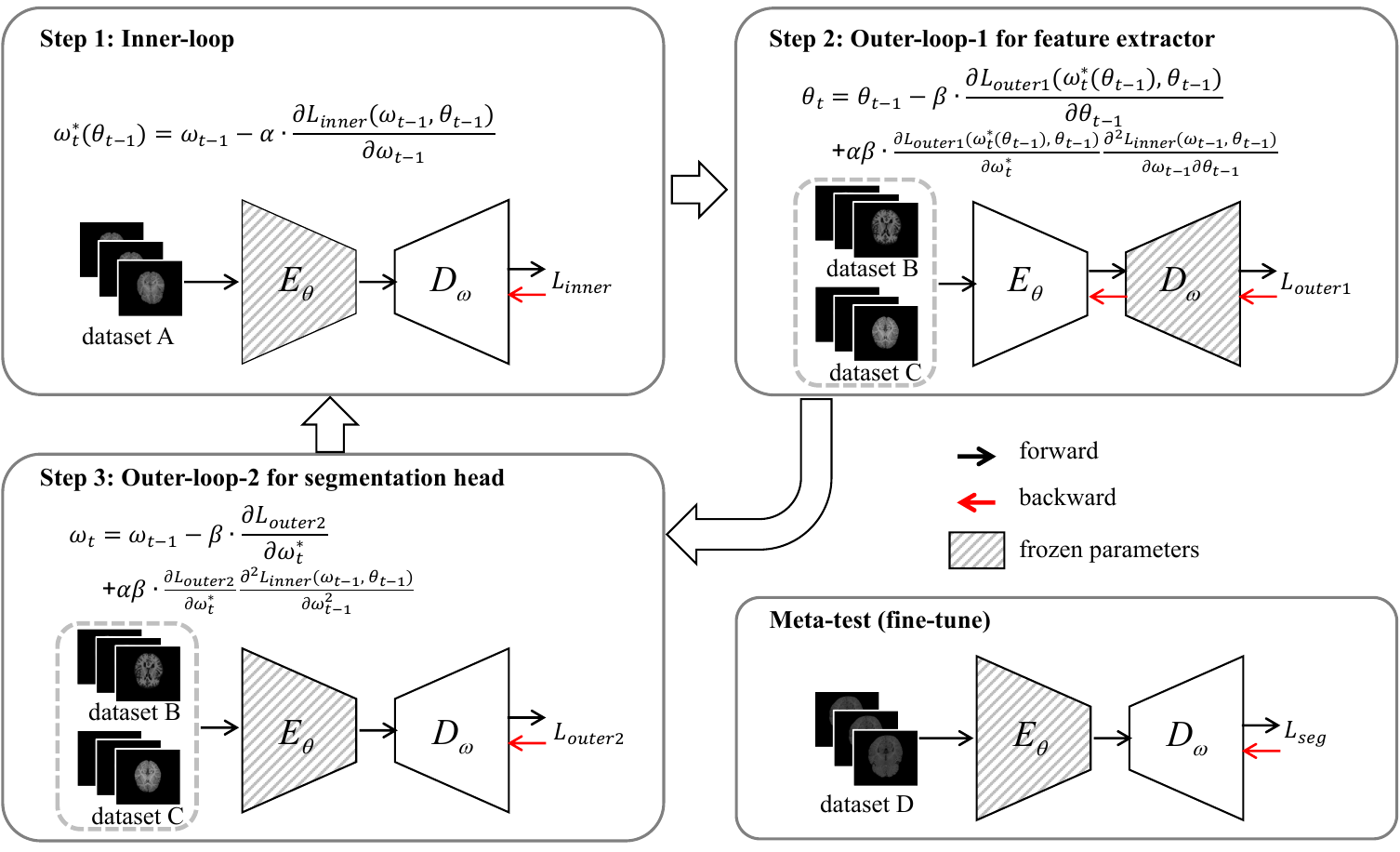} \caption{Schematic diagram of DuMeta++. The meta-training process is divided into three steps, leveraging a shared inner loop but leading to two distinct outer loops for learning a feature extractor and segmentation head. Specifically, steps (1) and (2) comprise the MFL procedure for the frozen feature extractor, while steps (1) and (3) realize the MIL procedure for establishing a well-initialized segmentation head. At inference time, we fine-tune only the segmentation head using data from the new, unseen domain.} \label{fig:method} \end{figure*}

\subsection{DuMeta++ Paradigm} We propose a dual meta-learning (DuMeta++) framework, illustrated in Fig.~\ref{fig:method}, to address the challenge of brain tissue segmentation across the lifespan. In the meta-training (pre-training) phase, DuMeta++ simultaneously learns (i) a universal, plug-and-play feature extractor $E_{\theta}(\cdot)$ that enforces longitudinally consistent anatomical representations, and (ii) a well-initialized segmentation head $D_{\omega}(\cdot)$. These two components are optimized via separate outer loops, each relying on a shared inner loop within a bilevel meta-learning scheme. During meta-testing (fine-tuning), we update only the segmentation head on a small number of labeled samples from an unseen age domain.

\subsubsection{Meta-Learning of the Universal Feature Extractor} \label{subsubsec:mfl} Our first objective is to learn $E_{\theta}(\cdot)$ such that it provides consistent, age-invariant representations across diverse time points. We treat $E_{\theta}(\cdot)$ as an explicit meta-learner and $D_{\omega}(\cdot)$ as its base-learner within a Meta-Feature Learning (MFL) paradigm.

Concretely, as shown in Fig.~\ref{fig:method}(1), for each meta-training iteration, we initially freeze $E_{\theta}(\cdot)$ while updating $D_{\omega}(\cdot)$ in the inner loop. We randomly sample one dataset from the meta-training pool and compute: 
\begin{equation}
    \omega^*_t(\theta_{t-1}) = \omega_{t-1} - \alpha \cdot\frac{\partial L_{\text{inner}}(\omega_{t-1}, \theta_{t-1})}{\partial \omega_{t-1}},
\label{eq:1}
\end{equation}
where $L_{\text{inner}}$ is the inner-loop loss (Dice + Cross-Entropy), $\alpha$ is the learning rate, and $\theta_{t-1}, \omega_{t-1}$ are the meta-learner and base-learner parameters at iteration $(t-1)$, respectively. As $L_{\text{inner}}$ also depends on $\theta_{t-1}$, the updated $\omega^*_t(\theta_{t-1})$ implicitly remains a function of $\theta_{t-1}$, thus forming the implicit relationship that will guide the next outer loop.

Since the inner loop can be interpreted as a mini fine-tuning step for the base-learner, we desire $E_{\theta}(\cdot)$ to be sufficiently robust so that, combined with a fine-tuned $D_{\omega}(\cdot)$, it yields strong segmentation outcomes. Hence, as depicted in Fig.~\ref{fig:method}(2), we select two additional datasets from the meta-training pool to refine $E_{\theta}(\cdot)$: 
\begin{equation}
\label{eq:2}
    \theta_t = \theta_{t-1} - \beta\cdot\frac{\partial L_{\text{outer}1}}{\partial\theta_{t-1}};
\end{equation}
\begin{align}
\label{eq:3}
    \frac{\partial L_{\text{outer}1}}{\partial\theta_{t-1}} = &\frac{\partial L_{\text{outer}1}( \omega^*_t(\theta_{t-1}), \theta_{t-1})}{\partial\theta_{t-1}} \nonumber\\
    + & \frac{\partial L_{\text{outer}1}( \omega^*_t(\theta_{t-1}), \theta_{t-1})}{\partial\omega^*_t} \frac{\partial\omega^*_t}{\partial\theta_{t-1}},
\end{align}
where $L_{\text{outer}1}$ comprises the segmentation objective (Dice + Cross-Entropy) and class-aware constraints (introduced in Sec.~\ref{subsec:regularization}). The gradient in Eq.(~\ref{eq:3}) contains both a \emph{direct} term minimizing $L_{\text{outer}1}$ and an \emph{indirect} term arising from the implicit relationship in Eq.(~\ref{eq:1}), ensuring that $E_{\theta}(\cdot)$ remains compatible with the updated base-learner for potential unseen domains. Explicitly expanding the indirect term: 
\begin{align}
\label{eq:4}
    \frac{\partial\omega^*_t}{\partial\theta_{t-1}} = & \frac{\partial(\omega_{t-1} - \alpha \cdot\frac{\partial L_{\text{inner}}(\omega_{t-1}, \theta_{t-1})}{\partial \omega_{t-1}})}{\partial\theta_{t-1}} \nonumber\\
    = &- \alpha \cdot\frac{\partial^2 L_{\text{inner}}(\omega_{t-1}, \theta_{t-1})}{\partial\omega_{t-1}\partial\theta_{t-1}}.
\end{align}
Substituting Eqs.(~\ref{eq:3}) and (~\ref{eq:4}) into Eq.(~\ref{eq:2}) yields: 
\begin{align}
\label{eq:5}
    \theta_t = & \theta_{t-1} - \beta\cdot\frac{\partial L_{\text{outer}1}(\omega^*_t(\theta_{t-1}), \theta_{t-1})}{\partial\theta_{t-1}}\nonumber\\
    + & \alpha\beta\cdot\frac{\partial L_{\text{outer}1}(\omega^*_t(\theta_{t-1}), \theta_{t-1})}{\partial\omega^*_t}\frac{\partial^2 L_{\text{inner}}(\omega_{t-1}, \theta_{t-1})}{\partial\omega_{t-1}\partial\theta_{t-1}},
\end{align}
thus reflecting both first-order (direct) and second-order (indirect) terms for updating the encoder.

\subsubsection{Meta-Learning of the Initialized Segmentation Head} \label{subsubsec:mil} Having determined the universal feature extractor $E_{\theta}(\cdot)$ as per Sec.~\ref{subsubsec:mfl}, we now learn a favorable initialization for $D_{\omega}(\cdot)$ through a Meta-Initialization Learning (MIL) scheme. In this phase, $E_{\theta}(\cdot)$ is held fixed, and the initial segmentation head weights, denoted $\phi_{t-1}$, act as an implicit meta-learner.

First, we initialize segmentation head $D_{\omega}(\cdot)$ with meta-learner parameter $\phi_{t-1}$
\begin{equation}
    \omega_{t-1} = \phi_{t-1}.
\end{equation}
A dataset is sampled to execute one inner-loop update: 
\begin{equation} \label{equ:inner-mil}
    \omega^*_t = \omega_{t-1} - \alpha \cdot\frac{\partial L_{\text{inner}}}{\partial \omega_{t-1}},
\end{equation}
after which we treat $\omega^*_t$ as a function of $\phi_{t-1}$. Keeping $E_{\theta}(\cdot)$ frozen, we then compute the loss on different datasets and back-propagate to update $\phi_{t-1}$ via: 
\begin{equation}
\label{eq:8}
    \phi_t = \phi_{t-1} - \beta\cdot\frac{\partial L_{\text{outer}2}}{\partial \phi_{t-1}};
\end{equation}
\begin{equation}
\label{eq:9}
    \frac{\partial L_{\text{outer}2}}{\partial \phi_{t-1}} = \frac{\partial L_{\text{outer}2}}{\partial\omega^*_t}\frac{\partial\omega^*_t}{\partial \phi_{t-1}};
\end{equation}
\begin{align}
\label{eq:10}
    \frac{\partial\omega^*_t}{\partial \phi_{t-1}} = & \frac{\partial(\omega_{t-1} - \alpha \cdot\frac{\partial L_{\text{inner}}}{\partial \omega_{t-1}})}{\partial \phi_{t-1}}\nonumber\\
    = & \mathbf{I} - \alpha\cdot\frac{\partial^2 L_{\text{inner}}}{\partial\phi_{t-1}^2},
\end{align}
Substituting Eqs.(~\ref{eq:9}) and(~\ref{eq:10}) back into Eq.~(~\ref{eq:8}) yields: 
\begin{equation}
    \phi_t = \phi_{t-1} - \beta\cdot\frac{\partial L_{\text{outer}2}}{\partial \omega_t^*}(\mathbf{I} - \alpha\cdot\frac{\partial^2 L_{\text{inner}}}{\partial\phi_{t-1}^2}).
\label{eq:11}
\end{equation}
In practice, second-order derivatives (i.e., $ \frac{\partial^2L_{\text{inner}}}{\partial \phi_{t-1}^2}$) can be omitted to limit complexity and stabilize training~\cite{nichol2018first}.

Both MFL (Sec.~\ref{subsubsec:mfl}) and MIL (Sec.~\ref{subsubsec:mil}) thus employ the same inner-loop updating mechanism (Eqs.(~\ref{eq:1}) and(~\ref{equ:inner-mil})), enabling us to unify them in a single shared procedure, as described in Algorithm~\ref{alg1}.

\begin{algorithm}[t] \caption{Dual Meta-Learning (DuMeta++)} \label{alg1} \begin{algorithmic}[1] \REQUIRE Meta-training pool (datasets A, B, and C), feature extractor $E_{\theta}(\cdot)$ with parameters $\theta$, and segmentation head $D_{\omega}(\cdot)$ with parameters $\omega$ \STATE Initialize $\theta$ and $\omega$ randomly \WHILE{not converged} \STATE \textbf{Inner loop:} \STATE Sample a mini-batch from one dataset \STATE Update $\omega$ using Eq.(~\ref{eq:1}) \STATE \textbf{Outer loop for the feature extractor:} \STATE Sample a mini-batch from the remaining two datasets \STATE Update $\theta$ using Eq.(~\ref{eq:5}) \STATE \textbf{Outer loop for the segmentation head:} \STATE Sample a mini-batch from the remaining two datasets \STATE Update $\omega$ using Eq.~(~\ref{eq:11}) \ENDWHILE \end{algorithmic} \end{algorithm}

\subsection{Class-Aware Regularization} \label{subsec:regularization} In Sec.~\ref{subsubsec:mfl}, the outer-loop loss $L_{\text{outer}1}$ for learning $E_{\theta}(\cdot)$ integrates a segmentation objective with a class-aware regularization term designed to support age-invariant anatomical representations. These regularizers, which encourage intra-tissue temporal alignment and inter-tissue spatial separation, stem from the core assumption that high-level differences among tissue types tend to remain consistent despite varying MRI appearances over time.

Unlike previous self-supervised learning methods~\cite{ren2022local, ouyang2021self} relying on image or patch-level contrastive objectives, we employ a \emph{weak-supervised} contrastive framework (i.e., according to pseudo labels) that aligns features of each tissue class directly---thereby simplifying the link to the final segmentation goal. Additionally, we apply these constraints at multiple feature scales within $D_{\omega}(\cdot)$ to optimize $E_{\theta}(\cdot)$ to achieve consistency across spatial resolutions.

Let $\{\mathbf{F}_k\}_{k=1}^K$  be the multi-scale feature maps from $D_{\omega}(\cdot)$, where each $\mathbf{F}_k \in \mathbb{R}^{BS \times NC_k \times H_k \times W_k \times D_k}$. Suppose we have tissue label maps $\mathbf{M}$ (or pseudo-labels), which are downsampled to match each $\mathbf{F}_k$, producing $\mathbf{M}_k$. Using these labels, we gather the voxel indices of specific tissues—gray matter (GM), white matter (WM), and cerebrospinal fluid (CSF)—and average their corresponding feature vectors to form $\mathbf{f}_k^{\text{GM}}$, $\mathbf{f}_k^{\text{WM}}$, and $\mathbf{f}_k^{\text{CSF}}$($\in \mathbb{R}^{BS \times NC_k}$).

Different from our previous work ~\cite{sun2023dual} which directly pushes and pulls sample-level features, we utilize memory-bank-based prototype features:
\begin{equation}
    \mathbf{p}_{k}^{\text{GM}} = \frac{1}{N} \sum_{i=0}^{N-1} \mathbf{f}_{k, i}^{\text{GM}}
\end{equation}
where $N$ is the memory bank capacity, and prototypes $\mathbf{p}_{k}^{\text{GM}}$, $\mathbf{p}_{k}^{\text{WM}}$, and $\mathbf{p}_{k}^{\text{CSF}}$ are aggregated from multiple mini-batches or datasets. We then employ a margin-based triplet loss to adaptively push or pull sample-level features relative to these prototypes. For instance, for dataset $B$ at scale $k$: 
\begin{equation}
\begin{split}
\mathcal{L}^{\text{GM}}_{B_k} = \max \Big( 0, \ 
& d(\mathbf{p}_{k}^{\text{GM}}, \mathbf{f}_{B_k}^{\text{GM}}) 
 - d(\mathbf{p}_{k}^{\text{GM}}, \mathbf{f}_{B_k}^{\text{CSF}}) \\
& - d(\mathbf{p}_{k}^{\text{GM}}, \mathbf{f}_{B_k}^{\text{WM}}) 
 + \lambda_1 \Big)
\end{split}
\label{eq:triplet}
\end{equation}
where $d(\mathbf{x},\mathbf{y}) = 1 - \frac{\mathbf{x}\cdot\mathbf{y}}{|\mathbf{x}||\mathbf{y}|}$ measures cosine distance, and $\lambda_1$ denotes the margin. Analogously, we define $\mathcal{L}^{\text{WM}}_{B_k}$, $\mathcal{L}^{\text{CSF}}_{B_k}$, and so forth for dataset $C$. Intuitively, this encourages representations of the same tissue to remain close across time, while pushing apart features belonging to different tissues. Samples that already lie within the correct distribution incur no additional gradient.

Summing these triplet-based terms over tissue types and scales for both datasets $B$ and $C$ produces the overall regularization loss: 
\begin{equation}
{L}_{reg} = \frac{1}{6K}\sum_{k=1}^{K} \left[ \mathcal{L}^{\text{GM}}_{B_k} + \mathcal{L}^{\text{WM}}_{B_k} + \mathcal{L}^{\text{CSF}}_{B_k} + \mathcal{L}^{\text{GM}}_{C_k} + \mathcal{L}^{\text{WM}}_{C_k} + \mathcal{L}^{\text{CSF}}_{C_k} \right] .
\end{equation}

The full outer-loop loss for meta-training the feature extractor thus becomes:
\begin{equation}
    L_{\text{outer1}} = L_{{seg}} + \lambda_2 L_{{reg}},
\end{equation}
where $L_{{seg}}$ is the segmentation loss (Dice + Cross-Entropy), and $\lambda_2$ is a constant (set to $0.1$ in our experiments).

\subsection{Convergence of the Dual Meta-Learning Pipeline} 
Formally, we perform a theoretical convergence analysis of our method for optimizing ${\omega}$ and ${\theta}$ through the bilevel optimization pipeline. The corresponding resutls show that DuMeta++ converges to critical points of both outer-loop loss functions under some mild conditions, as summarized in Theorems ~\ref{th1} and ~\ref{th2}, respectively. 
Due to space limitation, we cannot present detailed proofs of these two theorems, which can be found in the arXiv version.
%The proofs are presented in the supplementary material, available online.
\begin{Theorem} \label{th1}
Suppose the loss function $L_{\text{outer1}}$ is Lipschitz smooth with constant $L$, and the gradient of $\theta$ with respect to the loss function $L_{\text{outer1}}$ is Lipschitz continuous with constant $L$. Let the learning rate $\alpha_t, \beta_t, 1\leq t\leq T$ be monotonically descent sequences, and satisfy $\alpha_t=\min\{\frac{1}{L},\frac{c_1}{\sqrt{T}}\}, \beta_t=\min\{\frac{1}{L},\frac{c_2}{\sqrt{T}}\}$, for some $c_1,c_2>0$, such that $\frac{\sqrt{T}}{c_1}\geq L, \frac{\sqrt{T}}{c_2}\geq L$. Meanwhile, they satisfy $\sum_{t=1}^\infty \alpha_t = \infty,\sum_{t=1}^\infty \alpha_t^2 < \infty ,\sum_{t=1}^\infty \beta_t = \infty,\sum_{t=1}^\infty \beta_t^2 < \infty $. Then DuMeta++ can achieve $\mathbb{E}[ \|\nabla L_{\text{outer1}}(\omega_t(\theta_t))\|_2^2] \leq \epsilon$ in $\mathcal{O}(1/\epsilon^2)$ steps. More specifically,
	\begin{align}
		\min_{0\leq t \leq T} \mathbb{E}\left[ \left\|\nabla L_{\text{outer1}}(\omega_t(\theta_t))\right\|_2^2\right] \leq \mathcal{O}(\frac{C}{\sqrt{T}}),
	\end{align}
	where $C$ is a constant independent of the convergence process.
\end{Theorem}

\begin{Theorem} \label{th2}
Under the conditions of Theorem ~\ref{th1} and the gradient of $\phi$ with respect to the loss function $L_{\text{outer2}}$ is Lipschitz continuous with constant $L$, DuMeta++ can achieve $\mathbb{E}[ \|\nabla L_{\text{outer2}}(\omega_t(\phi_t))\|_2^2] \leq \epsilon$ in $\mathcal{O}(1/\epsilon^2)$ steps. More specifically,
	\begin{align}
		\min_{0\leq t \leq T} \mathbb{E}\left[ \left\|\nabla L_{\text{outer2}}(\omega_t(\phi_t))\right\|_2^2\right] \leq \mathcal{O}(\frac{C}{\sqrt{T}}),
	\end{align}
	where $C$ is a constant independent of the convergence process.
\end{Theorem}

\section{Experiment}
\subsection{Experimental Setup}
\subsubsection{Datasets} 
We assess the performance of our approach on the task of brain tissue segmentation in two particularly challenging age groups---infant and aging subjects---under a few-shot segmentation scenario. Specifically, the model (comprising both a feature extractor and a segmentation head) is meta-trained on three publicly available datasets: OASIS3, IBIS12M, and IBIS24M. It is then evaluated on two other public datasets, i.e., ADNI and iSeg-2019.

Pre-processing of T1-weighted (T1w) MRIs involves several steps: skull stripping, bias field correction, and intensity normalization. For high-quality pseudo tissue labels during meta-training, we apply the automated iBEAT pipeline~\cite{wang2023ibeat} to these pre-processed images. In the meta-test phase, iBEAT is similarly applied, followed by expert manual refinements to produce gold-standard annotations. Each meta-training dataset is divided into training and validation subsets at an 80--20 ratio, with the validation partition employed for model and hyperparameter selection. After completing meta-training, we pick one training subject from each meta-test dataset to fine-tune the segmentation head, and we report performance metrics on the respective test sets. 

The datasets used in this study are briefly summarized:
\begin{itemize}
\item \textbf{OASIS3}~\cite{lamontagne2019oasis}: This public dataset comprises 1,639 T1w MRIs from 992 individuals, each with 1 to 5 scans captured within a 5-year window. Participants include patients diagnosed with AD, healthy controls, and mild cognitive impairments, spanning the ages of 42 to 95.
\item \textbf{ADNI}~\cite{mueller2005Alzheimer}: A multi-site dataset of 2,389 longitudinal T1w MRIs (with at least two visits per subject). It covers healthy controls, individuals who may progress to AD, and AD patients, ages ranging from 20 to 90.
\item \textbf{IBIS}~\cite{hazlett2017early}: An infant imaging study that includes 1,272 T1w/T2w MRIs from 552 babies between 3 and 36 months of age, encompassing both typically developing infants and those at risk of autism spectrum disorder (ASD). We specifically focus on the 12-month (IBIS12M) and 24-month (IBIS24M) groups for meta-training.
\item \textbf{iSeg-2019}~\cite{9339962}: This MICCAI 2019 grand challenge dataset focuses on 6-month-old infant brain segmentation from multiple clinical sites. The training subset provides 10 subjects, while the test set contains 13. Each subject has both T1w and T2w scans. Segmenting 6-month-old brains is notably difficult due to the isointense phase, in which gray and white matter exhibit overlapping intensity ranges, especially near the cortical surface. 
%This overlap poses a significant barrier to accurate tissue classification.
\end{itemize}

% \usepackage{multirow}
\begin{table*}[]
\centering
\caption{One-shot segmentation results on the iSeg-2019 dataset.}
\label{tab:comp_iSeg2019_1shot}
\resizebox{\textwidth}{!}{
\begin{tabular}{l|ccc|ccc}
\hline
\multirow{2}{*}{Exp}           & \multicolumn{3}{c|}{Dice↑}                       & \multicolumn{3}{c}{ASD↓}                         \\ \cline{2-7} 
                               & CSF            & GM             & WM             & CSF            & GM             & WM             \\ \hline
% Fully-supervised               & 0.9664±0.0055  & 0.9370±0.0086  & 0.9190±0.0122  & 0.1056±0.0163  & 0.2877±0.0424  & 0.3285±0.0482  \\
RandInitUnet.2D                & 0.8767±0.0113 & 0.8402±0.0220 & 0.7965±0.0035 & 0.3317±0.0215 & 0.5578±0.0825 & 0.8479±0.1930 \\
RandInitUnet.3D                & 0.9029±0.0066 & 0.8616±0.0176 & 0.8200±0.0071 & 0.2652±0.0162 & 0.5497±0.0658 & 0.6487±0.1324 \\
Context Restore ~\cite{chen2019self}  & 0.9070±0.0065 & 0.8615±0.0118 & 0.8245±0.0155 & 0.2641±0.0239 & 0.5420±0.0560 & 0.6298±0.1068 \\
LNE ~\cite{ouyang2021self}            & 0.9315±0.0115 & 0.8937±0.0147 & 0.8654±0.0141 & 0.1753±0.0163 & 0.4482±0.0373 & 0.5274±0.0710 \\
GLCL ~\cite{chaitanya2020contrastive} & 0.9141±0.0103 & 0.8718±0.0185 & 0.8337±0.0088 & 0.2337±0.0147 & 0.5017±0.0546 & 0.6596±0.1024 \\
PCL ~\cite{zeng2021positional}        & 0.9139±0.0089 & 0.8688±0.0148 & 0.8313±0.0133 & 0.2326±0.0130 & 0.5161±0.0484 & 0.6692±0.0943 \\
PatchNCE ~\cite{park2020contrastive}  & 0.9444±0.0096 & 0.9043±0.0081 & 0.8782±0.0223 & 0.1360±0.0153 & 0.3820±0.0376 & 0.4520±0.0760 \\
MAML ~\cite{finn2017model} & 0.9433±0.0107                    & 0.9025±0.0100                    &  0.8768±0.0185                     & 0.1744±0.0139                    & 0.3931±0.0493                    & 0.5374±0.0859 \\
Reptile ~\cite{nichol2018first} & 0.9450±0.0129                    & 0.9047±0.0121                    &  0.8777±0.0143                     & 0.1173±0.0123                    &  0.3896±0.0451                    & 0.5324±0.0825 \\
LSOR ~\cite{ouyang2023lsor}                           & 0.9484±0.0107  & 0.9075±0.0098  & 0.8822±0.0182  & 0.1717±0.0143  & 0.3901±0.0489  & 0.5339±0.0861  \\
LSRL ~\cite{ren2022local}                           & 0.9508±0.0125  & 0.9102±0.0117  & 0.8823±0.0141  & 0.1147±0.0124  & 0.3860±0.0449  & 0.5294±0.0821  \\
DuMeta ~\cite{sun2023dual}                         & 0.9611±0.0059  & 0.9313±0.0083  & 0.9145±0.0126  & 0.1082±0.0158  & 0.2916±0.0423  & 0.3318±0.0483  \\
DuMeta++                      & \textbf{0.9665±0.0059}  & \textbf{0.9364±0.0080}  & \textbf{0.9196±0.0130}  & \textbf{0.1043±0.0162}  & \textbf{0.2887±0.0426}  & \textbf{0.3291±0.0479}  \\ \hline
\end{tabular}
}
\end{table*}

\subsubsection{Implementation Details} Our approach builds upon a 3D U-Net architecture~\cite{ronneberger2015u} with five down-sampling and five up-sampling blocks in the encoder and segmentation head, respectively. We employ InstanceNorm3d for normalization and ReLU as the activation function. Training is carried out using an SGD optimizer augmented with Nesterov momentum ($\mu = 0.99$) and a polynomial decay schedule, starting with a base learning rate of 0.01. A weight decay of $3 \times 10^{-5}$ is used to mitigate overfitting.
We crop the input MRI data into patches of size $128 \times 128 \times 128$, and train the model in mini-batches of two. To enhance multi-scale supervision, we introduce auxiliary outputs at five different scales, assigning scale-dependent weights of 0.0625, 0.125, 0.25, 0.5, and 1.0, from coarsest to finest level, respectively. Image augmentation broadly follows the nnU-Net pipeline~\cite{isensee2021nnu}, applying a variety of transformations to bolster generalization.
For performance evaluation, we track both the Dice similarity coefficient and the Average Symmetric Surface Distance (ASD). 
%A higher Dice score indicates improved spatial overlap between predictions and ground truth, while a lower ASD value reflects more accurate boundary alignment between the two.

\subsection{Comparison Results} 
We evaluate our method against approaches designed for longitudinally consistent representation learning, as well as meta-learning frameworks tailored to few-shot scenarios.
%需要简要地总结这些对比方法
These comepting methods include the pretext-task–based {Context Restore}~\cite{chen2019self}; contrastive learning-based {GLCL}~\cite{chaitanya2020contrastive}, {PCL}~\cite{zeng2021positional} and {PatchNCE}~\cite{park2020contrastive}; meta-learning baselines {MAML}~\cite{finn2017model} and {Reptile}~\cite{nichol2018first} for rapid adaptation in few-shot settings; and longitudinally consistent representation learning approaches {LNE}~\cite{ouyang2021self}, {LSOR}~\cite{ouyang2023lsor}, and {LSRL}~\cite{ren2022local}, which explicitly encourage longitudinal consistency for scans from the same subject.
All methods are pretrained and then finetuned with the segmentation task, and they share the same network backbone to ensure fairness. 

We assess performance on the ADNI and iSeg-2019 datasets under a challenging few-shot configuration during meta-testing. Both datasets pose substantial domain shifts due to progressive neurodegeneration or neurodevelopment.

\subsubsection{One-Shot Segmentation on iSeg-2019} The iSeg-2019 dataset comprises MRI scans of six-month-old infants from multiple sites, where the gray matter (GM) and white matter (WM) contrasts differ substantially from those of other age groups. To address this, our DuMeta++ model was meta-trained (pre-trained) using a 3D U-Net on IBIS12M and IBIS24M (corresponding to infants at 12 and 24 months), as well as OASIS3 (elderly subjects). The resulting model was then meta-tested (fine-tuned) on iSeg-2019, even though no six-month-old data had been used during meta-training. During meta-test, we froze the encoder and updated only the segmentation head with one sample from iSeg-2019.

Quantitative outcomes against competing methods are compiled in Table~\ref{tab:comp_iSeg2019_1shot}. As is evident, DuMeta++ achieves a clear performance margin, suggesting strong generalization to an unfamiliar time point. To complement these metrics, we provide qualitative examples in Fig.~\ref{fig:comp_iSeg}, which demonstrate how our approach captures structural details, notably cortical folding patterns, more accurately than the baselines.

\subsubsection{One-Shot Segmentation on ADNI} We further evaluated the pre-trained 3D U-Net on the ADNI dataset. Table~\ref{tab:comp_ADNI_1shot} and Fig.~\ref{fig:comp_adni} present the quantitative and qualitative results, respectively. Consistent to iSeg-2019, DuMeta++ outperforms other competitors when segmenting the aging brain, reinforcing its strong capability to handle age-related domain shifts.

Notably, other methods largely rely on self-supervision without ground-truth annotations, whereas our framework also does not require manual labels at meta-training, instead drawing on pseudo-labels generated via iBEAT in a semi-supervised manner. Incorporating pseudo labels helps prevent the mode collapse often observed in pure contrastive learning: due to U-Net skip connections, a self-supervised model can degenerate into trivial identity mappings that fail to encode useful semantics. In contrast, our cross-sectional strategy does not demand longitudinally paired samples, thus offering a practical solution for real-world datasets where longitudinal consistency is challenging to obtain.

% \usepackage{multirow}
\begin{table*}[]
\centering
\caption{One-shot segmentation results on the ADNI dataset.}
\label{tab:comp_ADNI_1shot}
\resizebox{\textwidth}{!}{
\begin{tabular}{l|ccc|ccc}
\hline
\multirow{2}{*}{Exp}           & \multicolumn{3}{c|}{Dice↑}                                                  & \multicolumn{3}{c}{ASD↓}                                                    \\ \cline{2-7} 
                               & CSF                     & GM                      & WM                      & CSF                     & GM                      & WM                      \\ \hline
% Fully-supervised               & 0.9813±0.0021          & 0.9684±0.0033          & 0.9800±0.0023          & 0.0215±0.0040          & 0.0307±0.0051          & 0.0314±0.0054          \\
RandInitUnet.2D                & 0.9223±0.0084 & 0.9051±0.0098 & 0.9361±0.0096 & 0.1571±0.0422 & 0.1573±0.0338 & 0.2059±0.0909 \\
RandInitUnet.3D                & 0.9459±0.0062 & 0.9236±0.0073 & 0.9495±0.0058 & 0.0933±0.0205 & 0.1160±0.0297 & 0.1226±0.0291 \\
Context Restore ~\cite{chen2019self}  & 0.9500±0.0058 & 0.9289±0.0069 & 0.9527±0.0054 & 0.0821±0.0179 & 0.1056±0.0267 & 0.1101±0.0241 \\
LNE ~\cite{ouyang2021self}            & 0.9665±0.0043 & 0.9456±0.0065 & 0.9637±0.0047 & 0.0487±0.0107 & 0.0711±0.0168 & 0.0760±0.0180 \\
GLCL ~\cite{chaitanya2020contrastive} & 0.9531±0.0055 & 0.9332±0.0069 & 0.9559±0.0052 & 0.0736±0.0151 & 0.0964±0.0238 & 0.1018±0.0236 \\
PCL ~\cite{zeng2021positional}        & 0.9526±0.0057 & 0.9314±0.0068 & 0.9547±0.0053 & 0.0750±0.0137 & 0.0995±0.0255 & 0.1059±0.0273 \\
PatchNCE ~\cite{park2020contrastive}  & 0.9716±0.0040 & 0.9520±0.0060 & 0.9687±0.0043 & 0.0402±0.0091 & 0.0591±0.0118 & 0.0634±0.0157 \\ 
MAML ~\cite{finn2017model} & 0.9680±0.0062                    & 0.9472±0.0070                    & 0.9647±0.0061                     & 0.0460±0.0197                    & 0.0680±0.0239                    & 0.0736±0.0281 \\
Reptile ~\cite{nichol2018first} & 0.9687±0.0056                    & 0.9476±0.0071                    &  0.9656±0.0049                     & 0.0452±0.0167                    & 0.0666±0.0163                    &  0.0702±0.0176 \\
LSOR ~\cite{ouyang2023lsor}                          & 0.9738±0.0067           & 0.9526±0.0074           & 0.9697±0.0064           & 0.0406±0.0194           & 0.0628±0.0234           & 0.0680±0.0286           \\
LSRL ~\cite{ren2022local}                          & 0.9744±0.0071           & 0.9535±0.0075           & 0.9706±0.0048           & 0.0398±0.0171           & 0.0621±0.0158           & 0.0645±0.0180           \\
DuMeta  ~\cite{sun2023dual}                       & 0.9809±0.0021          & 0.9678±0.0034          & 0.9796±0.0024          & 0.0222±0.0042          & 0.0315±0.0052          & 0.0322±0.0054          \\
DuMeta++                      & \textbf{0.9869±0.0024} & \textbf{0.9731±0.0036} & \textbf{0.9847±0.0023} & \textbf{0.0204±0.0046} & \textbf{0.0301±0.0057} & \textbf{0.0307±0.0051} \\ \hline
\end{tabular}
}
\end{table*}

\begin{figure*}[t]
  \centering
   \includegraphics[width=0.99\linewidth]{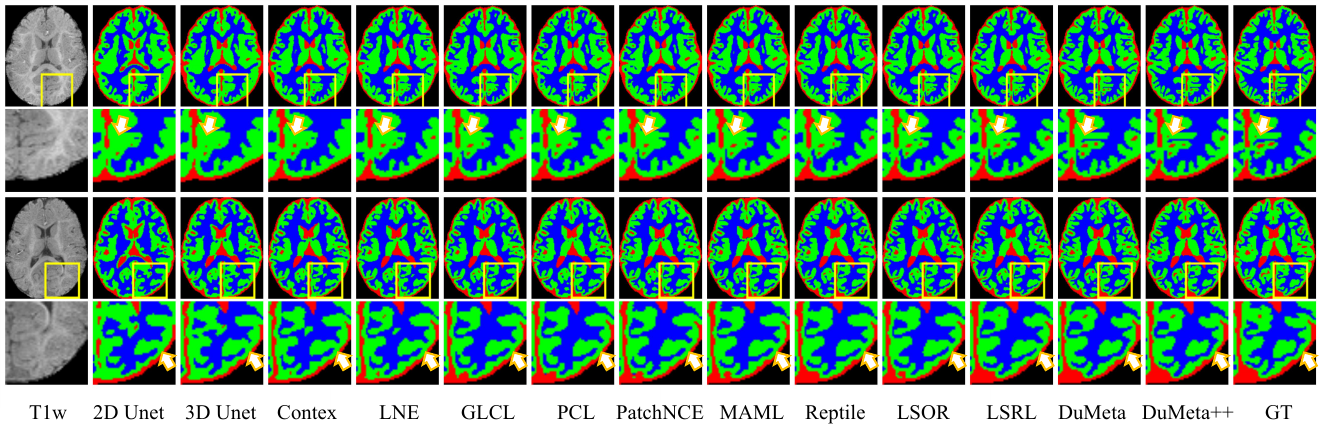}
   \caption{The 2D slice views of representative one-shot segmentation results on the held-out test set of iSeg-2019.}
\label{fig:comp_iSeg}
\end{figure*}

\begin{figure*}[t]
  \centering
   \includegraphics[width=0.99\linewidth]{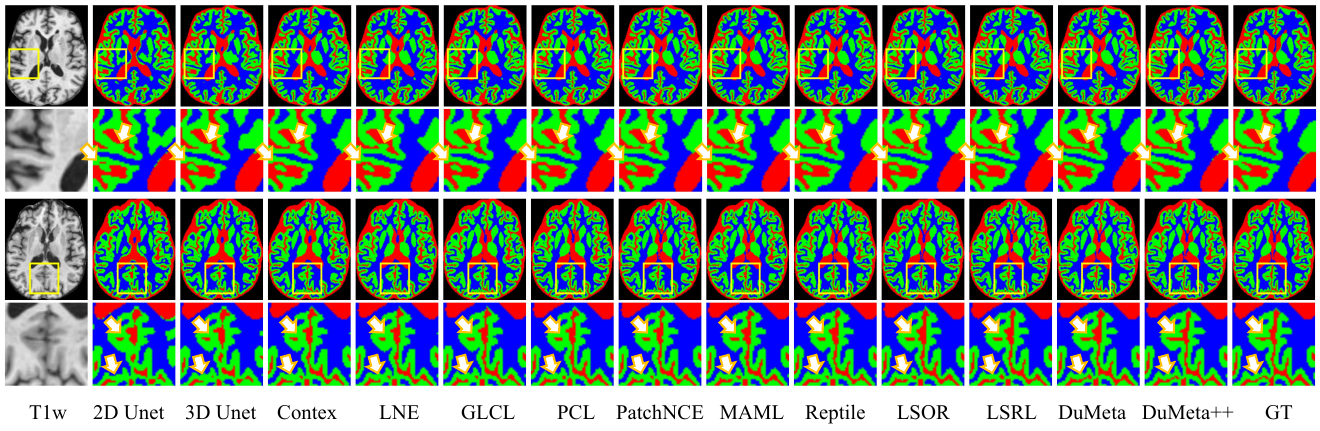}
   \caption{The 2D slice views of representative one-shot segmentation results on the held-out test set of ADNI.}
\label{fig:comp_adni}
\end{figure*}

\begin{figure}[t]
  \centering
   \includegraphics[width=0.99\linewidth]{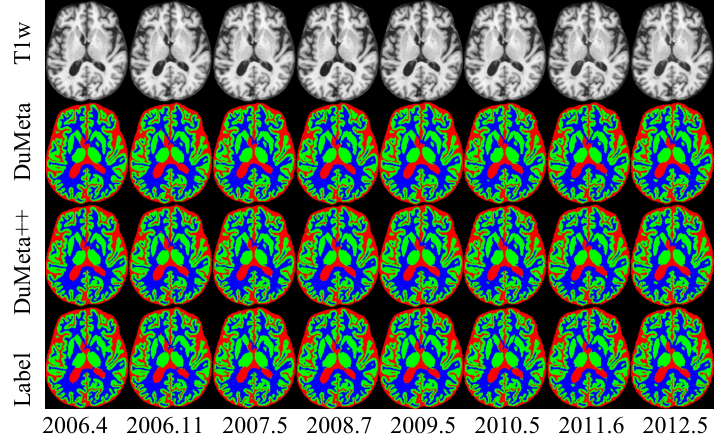}
   \caption{Visualization of segmentation results of the same subject at multiple time points on ADNI.}
\label{fig:comp_long}
\end{figure}

\subsubsection{Five-Shot Segmentation} We further evaluated the five-shot setting on the ADNI dataset using the same pre-trained 3D U-Net, with the corresponding quantitative results detailed in Table~\ref{tab:comp_ADNI_5shot}. Mirroring the trends observed in the one-shot experiments, DuMeta++ again achieved consistently higher accuracy for aging brain tissue segmentation, thereby reinforcing its ability to generalize effectively across the lifespan.

% \usepackage{multirow}
\begin{table*}[]
\centering
\caption{Five-shot segmentation results on the ADNI dataset.}
\label{tab:comp_ADNI_5shot}
\resizebox{\textwidth}{!}{
\begin{tabular}{l|ccc|ccc}
\hline
\multirow{2}{*}{Exp}           & \multicolumn{3}{c|}{Dice↑}                       & \multicolumn{3}{c}{ASD↓}                         \\ \cline{2-7} 
                               & CSF            & GM             & WM             & CSF            & GM             & WM             \\ \hline
% Fully-supervised               & 0.9813±0.0021 & 0.9684±0.0033 & 0.9800±0.0023 & 0.0215±0.0040 & 0.0307±0.0051 & 0.0314±0.0054 \\
RandInitUnet.2D                & 0.9371±0.0073 & 0.9186±0.0077 & 0.9467±0.0064 & 0.1123±0.0252 & 0.1278±0.0313 & 0.1361±0.0291 \\
RandInitUnet.3D                & 0.9552±0.0054 & 0.9344±0.0068 & 0.9566±0.0051 & 0.0693±0.0126 & 0.0932±0.0238 & 0.0985±0.0247 \\
Context Restore~\cite{chen2019self}  & 0.9559±0.0054 & 0.9365±0.0066 & 0.9583±0.0051 & 0.0687±0.0165 & 0.0899±0.0242 & 0.0936±0.0219 \\
LNE~\cite{ouyang2021self}            & 0.9708±0.0040 & 0.9511±0.0061 & 0.9681±0.0043 & 0.0414±0.0094 & 0.0609±0.0132 & 0.0641±0.0145 \\
GLCL~\cite{chaitanya2020contrastive} & 0.9622±0.0047 & 0.9418±0.0064 & 0.9615±0.0046 & 0.0558±0.0108 & 0.0781±0.0174 & 0.0824±0.0179 \\
PCL~\cite{zeng2021positional}        & 0.9575±0.0054 & 0.9370±0.0069 & 0.9581±0.0051 & 0.0648±0.0121 & 0.0878±0.0208 & 0.0949±0.0245 \\
PatchNCE~\cite{park2020contrastive}  & 0.9741±0.0038  & 0.9553±0.0060  & 0.9708±0.0042  & 0.0405±0.0087 & 0.0591±0.0112 & 0.0641±0.0160 \\
MAML ~\cite{finn2017model}                          & 0.9717±0.0066  & 0.9500±0.0074  & 0.9677±0.0058  & 0.0434±0.0194  & 0.0609±0.0123  & 0.0709±0.0281  \\
Reptile ~\cite{nichol2018first}                       & 0.9721±0.0057  & 0.9504±0.0072  & 0.9694±0.0045  & 0.0419±0.0165  & 0.0640±0.0164  & 0.0670±0.0180  \\
LSOR ~\cite{ouyang2023lsor}                          & 0.9771±0.0064  & 0.9566±0.0077  & 0.9735±0.0068  & 0.0379±0.0189  & 0.0591±0.0235  & 0.0647±0.0283  \\
LSRL ~\cite{ren2022local}                          & 0.9774±0.0070  & 0.9568±0.0074  & 0.9744±0.0047  & 0.0369±0.0173  & 0.0596±0.0161  & 0.0615±0.0177  \\
DuMeta  ~\cite{sun2023dual}                       & 0.9842±0.0026  & 0.9704±0.0029  & 0.9827±0.0027  & 0.0205±0.0038  & 0.0306±0.0053  & 0.0303±0.0053  \\
DuMeta++                      & \textbf{0.9900±0.0028}  & \textbf{0.9767±0.0033}  & \textbf{0.9872±0.0024}  & \textbf{0.0190±0.0042}  & \textbf{0.0292±0.0062}  & \textbf{0.0296±0.0056}  \\ \hline
\end{tabular}
}
\end{table*}

\subsubsection{Evaluation of Longitudinal Consistency} We further examined how consistently each method performs over time on ADNI by employing spatio-temporal consistency of segmentation (STCS)~\cite{li2021longitudinal} and absolute symmetrized percent change (ASPC)~\cite{reuter2012within} as evaluation metrics. From Table~\ref{tab:comp_lc}, it is evident that DuMeta++ demonstrates superior longitudinal stability. To gain additional insights into the underlying representations, we visualized the meta-learned features using t-SNE, with examples in Fig.~\ref{fig:t-sne} indicating that DuMeta++ learned robust, time-invariant feature embeddings. Moreover, Fig.~\ref{fig:comp_long} shows segmentation outcomes from a single subject at multiple time points, illustrating how our approach consistently maintains accurate segmentation performance across different ages.

\begin{figure}
    \centering
    \includegraphics[width=0.8\linewidth]{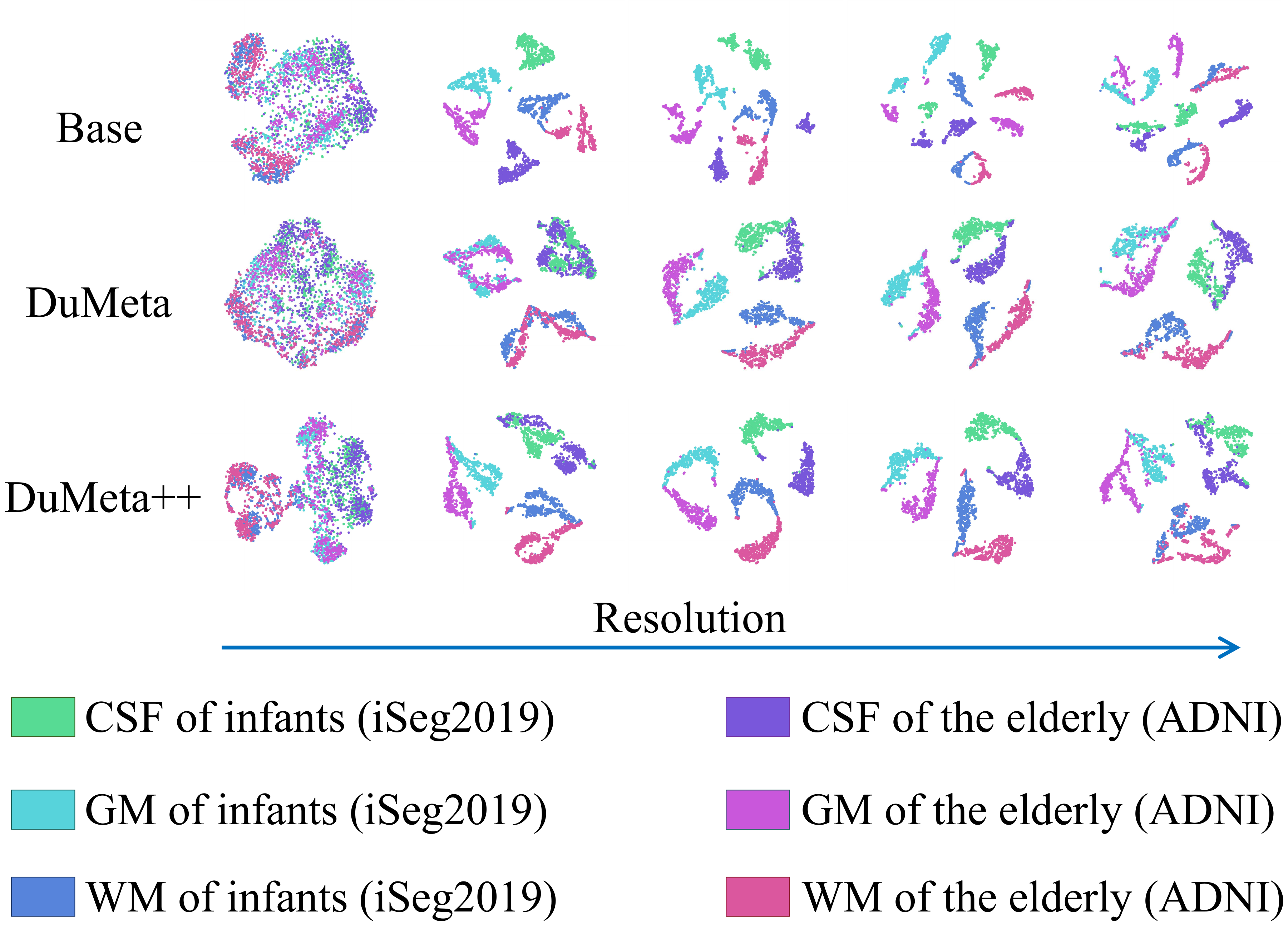}
    \caption{Visualizing meta-learned features for different age groups. The same row from left to right is the low-resolution to high-resolution features of the segmentation head.}
    \label{fig:t-sne}
\end{figure}

% \usepackage{multirow}
\begin{table}[]
\centering
\caption{Evaluation of longitudinal consistency on the ADNI dataset.}
\label{tab:comp_lc}
\setlength{\tabcolsep}{4pt} 
\resizebox{\columnwidth}{!}{%
\begin{tabular}{l|ccc|ccc}
\hline
\multirow{2}{*}{Exp} & \multicolumn{3}{c|}{STCS↑} & \multicolumn{3}{c}{ASPC↓} \\ \cline{2-7} 
                     & CSF     & GM      & WM     & CSF    & GM      & WM     \\ \hline
% Fully-supervised     & 0.9075  & 0.8872  & 0.9197 & 4.94   & 5.90    & 4.07   \\
RandInitUnet.2D      & 0.8809          & 0.8584          & 0.8998          & 6.35                    & 8.32                   & 5.16                   \\
RandInitUnet.3D      & 0.9063          & 0.8862          & 0.9177          & 5.06                    & 6.02                   & 4.20                   \\
Context Restore ~\cite{chen2019self}     & 0.8386          & 0.8180          & 0.8580          & 9.15                    & 11.07                  & 8.08                   \\
LNE ~\cite{ouyang2021self}                 & 0.9015          & 0.8796          & 0.9132          & 4.78                    & 6.81                   & 4.49                   \\
GLCL ~\cite{chaitanya2020contrastive}                & 0.8956          & 0.8706          & 0.9057          & 5.18                    & 7.57                   & 4.75                   \\
PCL ~\cite{zeng2021positional}                 & 0.9155          & 0.8968          & 0.9275          & 3.98                    & 5.63                   & 3.13                   \\
PatchNCE ~\cite{park2020contrastive}            & 0.9250          & 0.9087          & 0.9356          & 3.80                    & 5.05                   & 2.94                   \\
MAML ~\cite{finn2017model}                & 0.9027  & 0.8812  & 0.9146 & 4.65   & 6.71    & 4.32   \\
Reptile ~\cite{nichol2018first}             & 0.8966  & 0.8722  & 0.9067 & 5.06   & 7.47    & 4.59   \\
LSOR ~\cite{ouyang2023lsor}                & 0.9171  & 0.8979  & 0.9287 & 3.84   & 5.51    & 2.99   \\
LSRL ~\cite{ren2022local}                 & 0.9264  & 0.9107  & 0.9368 & 3.65   & 4.85    & 2.76   \\
DuMeta ~\cite{sun2023dual}              & 0.9373  & 0.9222  & 0.9468 & 2.66   & 3.40    & 2.20   \\
DuMeta++            & \textbf{0.9393}  & \textbf{0.9231}  & \textbf{0.9474} & \textbf{2.56}   & \textbf{3.20}    & \textbf{2.02}   \\ \hline
\end{tabular}
}
\end{table}

\subsection{Ablation Studies} We carried out a set of ablation experiments on the iSeg-2019 training set, which includes 10 available labeled samples. In the meta-test phase, one of these samples was chosen for fine-tuning, two were assigned to validation, and the remaining two served as test subjects.

% \usepackage{multirow}
\begin{table}[]
\caption{Ablation study of different components on the held-out test set of iSeg-2019.}
\label{tab:abla_component}
\resizebox{\columnwidth}{!}{%
\begin{tabular}{lcccc|ccc}
\hline
\multicolumn{1}{c}{\multirow{2}{*}{Exp}} & \multirow{2}{*}{Feature} & \multirow{2}{*}{Loss} & \multirow{2}{*}{MFL} & \multirow{2}{*}{MIL} & \multicolumn{3}{c}{Dice↑}                          \\ \cline{6-8} 
\multicolumn{1}{c}{}                     &                           &                           &                      &                      & CSF             & GM              & WM              \\ \hline
A                                        &                           &                           &                     &                      & 0.9155          & 0.8724          & 0.8375          \\
B                                        &                           &                           & $\checkmark$                    &                      & 0.9332          & 0.8927          & 0.8651          \\
C                                        &                           &                           & $\checkmark$                    & $\checkmark$                    & 0.9358          & 0.8955          & 0.8679          \\
D                                        & Prototype                         & Sum                          & $\checkmark$                    & $\checkmark$                    & 0.9456          & 0.9049          & 0.8777          \\
E                                        & Sample                          & Triplet                         & $\checkmark$                    & $\checkmark$                    & 0.9452          & 0.9017          & 0.8727          \\
F                                        & Prototype                         & Triplet                         & $\checkmark$                    & $\checkmark$                    & \textbf{0.9499} & \textbf{0.9092} & \textbf{0.8832} \\ \hline
\end{tabular}
}
\end{table}

\subsubsection{Role of DuMeta++ and Class-Aware Regularization} To investigate the importance of our proposed dual meta-learning strategy (DuMeta++) and the class-aware regularization terms, we systematically removed or modified these components to observe the impact on segmentation accuracy. Table~\ref{tab:abla_component} summarizes the results. Relative to the baseline (A in Table~\ref{tab:abla_component}), incorporating meta-feature learning (MFL, labeled B) led to a notable performance increase, underscoring its capacity to learn representations that generalize well across different ages. Further improvements were observed upon adding meta-initialization learning (MIL, labeled C).

In contrast, model D eliminates the max operator in Eq.~\ref{eq:triplet} and merely sums and subtracts distances, whereas model E substitutes sample-level features for prototypes. Both variations performed less effectively than our final model (F), which applies a memory-bank-based prototype approach and margin-based triplet loss. Hence, combining these two components (F) yields the best accuracy.

\begin{figure}
    \centering
    \includegraphics[width=0.8\linewidth]{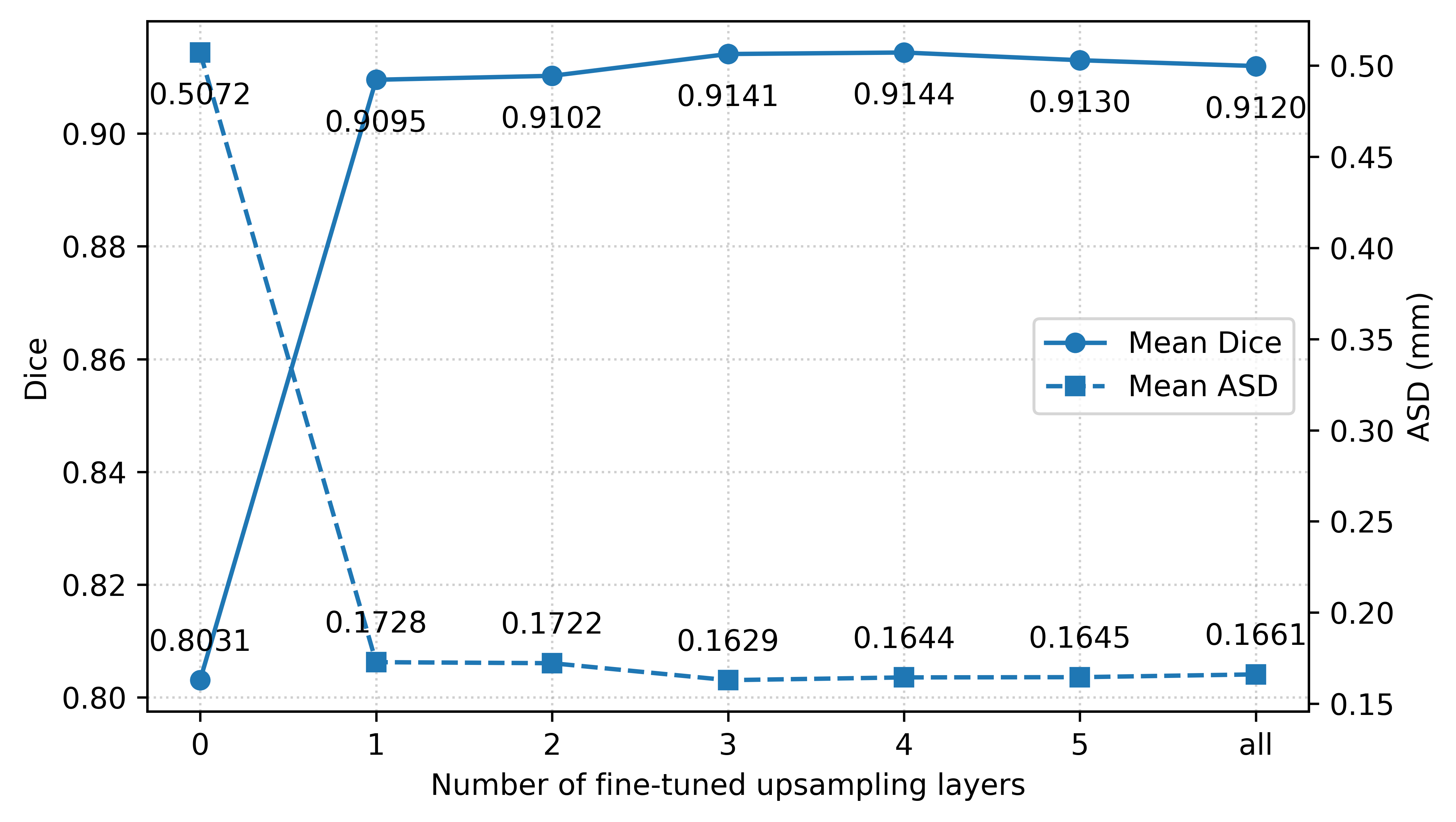}
    \caption{Ablation study of different base/meta-learner splits on the held-out test set of iSeg-2019.}
    \label{fig:abla_decoder}
\end{figure}

\subsubsection{Influence of Different Base/Meta-Learner Splits} We additionally examined how different divisions of the 3D U-Net into base learner (segmentation head) and meta-learner (feature extractor) affect performance. As shown in Fig.~\ref{fig:abla_decoder}, simply applying the pre-trained model (fine-tuning 0 layer) yields suboptimal outcomes. Fine-tuning all layers significantly boosts performance, indicating genuine domain gaps across the age spectrum. However, tuning only a small number of parameters (e.g., 1 upsampling layer) proved insufficient, likely because it restricts the model to a local, rather than global, optimum. Notably, fine-tuning the last three layers achieved the best compromise between ease of convergence and accuracy. It implies that maintaining most of the pre-trained encoder while focusing adaptation on a moderate subset of segmentation head layers yields strong, generalizable segmentation.

\begin{figure}
    \centering
    \includegraphics[width=0.8\linewidth]{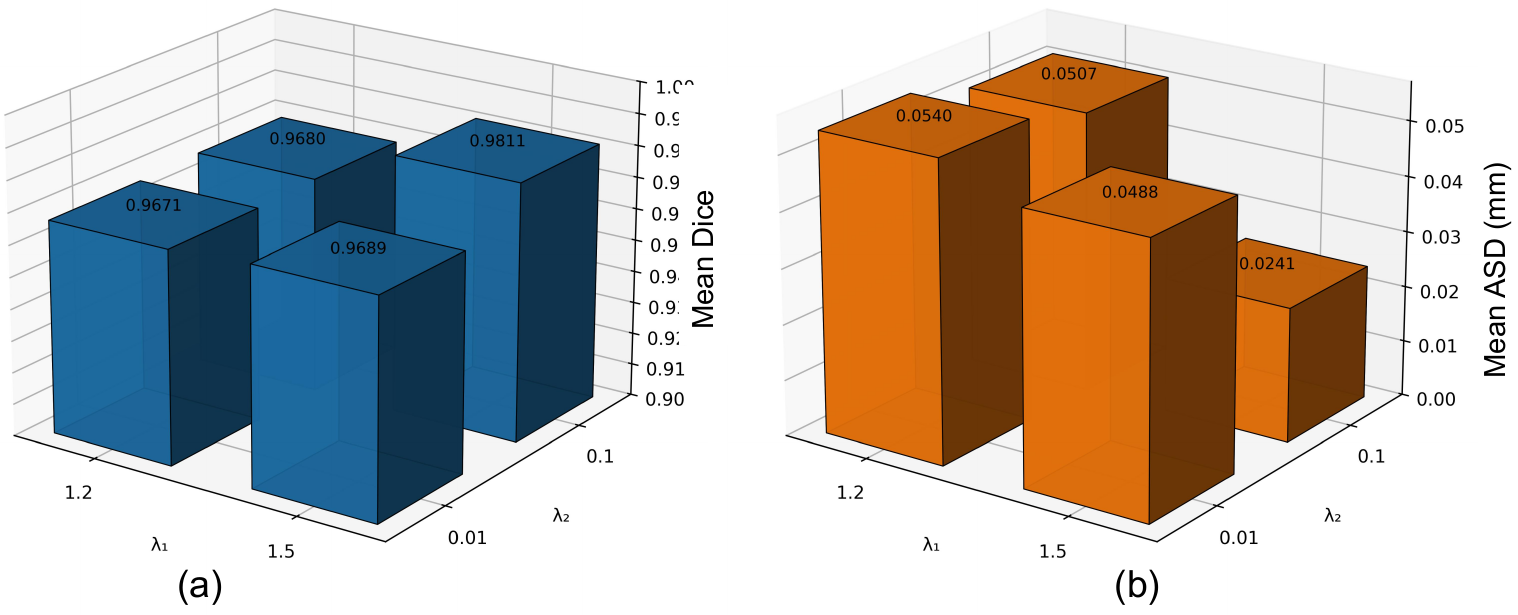}
    \caption{Ablation of hyperparameters $\lambda_1$ and $\lambda_2$ on the ADNI dataset. (a) Mean Dice; (b) Mean ASD.}
    \label{fig:abla_hyper}
\end{figure}

\subsubsection{Hyperparameter Sensitivity Analysis} Fig.~\ref{fig:abla_hyper} summarizes the performance variation on ADNI under different values of the triplet loss margin ($\lambda_1$) and the regularization weight ($\lambda_2$). Overall, setting $\lambda_1 = 1.5$ and $\lambda_2 = 0.1$ yields the highest mean Dice score (0.9811) of all tissue types, along with the lowest mean ASD values (0.0241). This configuration indicates that a moderately larger triplet margin ($\lambda_1$), 
together with a moderate regularization weight ($\lambda_2$), 
produces more precise boundaries and better class separation. By contrast, using a lower $\lambda_1$ (e.g., 1.2) or a smaller $\lambda_2$ (e.g., 0.01) generally worsens both metrics, suggesting that the balance between the learning dynamics and regularization strength is pivotal for stable and accurate segmentation.

% \usepackage{multirow}
\begin{table}[]
\caption{Ablation study of memory bank capacity on the ADNI dataset.}
\label{tab:abla_capacity}
\resizebox{\columnwidth}{!}{%
\begin{tabular}{l|ccc|ccc}
\hline
\multirow{2}{*}{Capacity} & \multicolumn{3}{c|}{Dice↑}                          & \multicolumn{3}{c}{ASD↓}                            \\ \cline{2-7} 
                          & CSF             & GM              & WM              & CSF             & GM              & WM              \\ \hline
1                         & 0.9810          & 0.9679          & 0.9797          & 0.0220          & 0.0313          & 0.0321          \\
10                        & 0.9847          & 0.9710          & 0.9822          & \textbf{0.0200} & 0.0302          & 0.0323          \\
100                       & \textbf{0.9869} & 0.9731          & 0.9847          & 0.0204          & 0.0301          & 0.0307          \\
1000                      & 0.9867          & \textbf{0.9734} & \textbf{0.9849} & 0.0203          & \textbf{0.0293} & \textbf{0.0297} \\ \hline
\end{tabular}
}
\end{table}

\subsubsection{Memory Bank Capacity Test} We also explored how varying the capacity of the memory bank affects accuracy. As shown in Table~\ref{tab:abla_capacity}, while a capacity of 1000 attains the best scores on several entries, a capacity of 100 offers a strong accuracy-efficiency trade-off.

\section{Discussion \& Conclusion} 

While DuMeta++ demonstrates promising results, several limitations warrant future investigation. The reliance on pseudo-labels during meta-training, though practical, may introduce noise and cohort biases; mitigating strategies such as noise-aware training and periodic prototype auditing could enhance robustness. Furthermore, the prototype-centric regularization with a fixed memory bank may not fully encapsulate intra-class variability, indicating potential for multi-prototype or adaptively margined formulations. The current few-shot protocol could be extended with lightweight encoder adjustments or test-time adaptation strategies to handle more severe domain shifts. Finally, while our experiments validate the framework on T1-weighted MRI for three tissue classes across several datasets, future work should assess its generality on multi-contrast data, diverse populations, and clinically oriented endpoints to better establish its translational readiness.

In conclusion, we have presented DuMeta++, a unified dual meta-learning framework for generalizable, few-shot brain tissue segmentation across varying age groups. The method jointly learns an age-invariant feature extractor via meta-feature learning and a lightweight adaptable segmentation head via meta-initialization learning, reinforced by a memory-bank-based class-aware regularization that enforces longitudinal consistency without paired data. Extensive evaluations under few-shot settings demonstrate that DuMeta++ achieves superior performance compared to existing state-of-the-art methods in cross-age generalization.

{\small
\bibliographystyle{IEEEtran}
\bibliography{egbib}
}

\newpage
\onecolumn 
\appendices

\section{Convergence Proof of Proposed DuMeta++ Learning Algorithm}
% \begin{equation}
%     P_{k, t}^{GM} = \alpha \cdot P_{k, t - 1}^{GM} + (1 - \alpha) \cdot f_{k, t - 1}^{GM}
% \end{equation}
% \begin{equation*}
%     \mathcal{L}^{GM}_{\text{triplet}} = \left[ \left\| P_{k, t}^{GM} - f_{k, t}^{GM} \right\|^2 
%      - \left\| P_{k, t}^{GM} - f_{k, t}^{WM} \right\|^2 - \left\| P_{k, t}^{GM} - f_{k, t}^{CSF} \right\|^2 + \alpha \right]_+
% \end{equation*}

% \begin{equation}
%     P_{k, t}^{GM} = \frac{1}{N} \sum_{i=0}^{N-1} f_{k, t - i}^{GM}
% \end{equation}
This section provides the proofs of Theorems 1 and 2 in the paper.

Suppose that we have a dataset with $M$ samples $\{(x_i^{(m)},y_i^{(m)}),1\leq i\leq M\}$ with clean labels for the outer-loop, and the overall meta loss is,
\begin{align}
	L_{\text{outer1}}(\omega^*(\theta))=\frac{1}{M} \sum_{i=1}^M L_{\text{outer1}}^i(\omega^*(\theta)),
\end{align}
where $\omega^*$ is the parameter of the feature extractor, and $\theta$ is the parameter of the segmentation head. Let's suppose we have another $N$ training data, $\{(x_i,y_i),1\leq i \leq N\}$, where $M\ll N$, and the overall training loss is,

\begin{align}\label{eqob}
	L_{\text{inner}}(\omega;\theta) = \frac{1}{N} \sum_{i=1}^N L_{\text{inner}}^i(\omega),
\end{align}.

\begin{Theorem} \label{ths1}
Suppose the loss function $L_{\text{inner}}$ is Lipschitz smooth with constant $L$, and the gradient of $\theta$ with respect to the loss function $L_{\text{outer1}}$ is Lipschitz continuous with constant $L$. Let the learning rate $\alpha_t, \beta_t, 1\leq t\leq T$ be monotonically descent sequences, and satisfy $\alpha_t=\min\{\frac{1}{L},\frac{c_1}{\sqrt{T}}\}, \beta_t=\min\{\frac{1}{L},\frac{c_2}{\sqrt{T}}\}$, for some $c_1,c_2>0$, such that $\frac{\sqrt{T}}{c_1}\geq L, \frac{\sqrt{T}}{c_2}\geq L$. Meanwhile, they satisfy $\sum_{t=1}^\infty \alpha_t = \infty,\sum_{t=1}^\infty \alpha_t^2 < \infty ,\sum_{t=1}^\infty \beta_t = \infty,\sum_{t=1}^\infty \beta_t^2 < \infty $. Then DuMeta++ can achieve $\mathbb{E}[ \|\nabla L_{\text{outer1}}({\omega}^*_t(\theta_t))\|_2^2] \leq \epsilon$ in $\mathcal{O}(1/\epsilon^2)$ steps. More specifically,
	\begin{align}
		\min_{0\leq t \leq T} \mathbb{E}\left[ \left\|\nabla L_{\text{outer1}}({\omega}^*_t(\theta_t))\right\|_2^2\right] \leq \mathcal{O}(\frac{C}{\sqrt{T}}),
	\end{align}
	where $C$ is some constant independent of the convergence process.
\end{Theorem}

\begin{proof}
	The update equation of $\theta$ in each iteration is as follows:
	\begin{align*}
		\theta_{t+1} =  \theta_t -\beta_t \frac{1}{m}\sum_{i=1}^{m} \nabla_{ \theta} L_{\text{outer1}}^i({\omega}^* _{t+1}(\theta))\Big|_{\theta_t}.
	\end{align*}
	
Under the sampled mini-batch $\Xi_t$, the updating equation can be rewritten as:
	\begin{align*}
		\theta_{t+1} =  \theta_t -\beta_t \nabla_{ \theta}L_{\text{outer1}}({\omega}^*_{t+1}(\theta_t))\big|_{\Xi_t}.
	\end{align*}
	Since the mini-batch is drawn uniformly from the entire data set, the above update equation can be written as:
	\begin{align*}
		\theta_{t+1} =  \theta_t -\beta_t[ \nabla_{ \theta}L_{\text{outer1}}({\omega}^*_{t+1}(\theta_t))+\xi_t],
	\end{align*}
	where $\xi_t = \nabla_{ \theta}L_{\text{outer1}}({\omega}^*_{t+1}(\theta_t))\big|_{\Xi_t}-\nabla_{ \theta}L_{\text{outer1}}({\omega}^*_{t+1}(\theta_t))$. Note that $\xi_t$ are i.i.d random variable with finite variance, since $\Xi_t$ are drawn i.i.d with a finite number of samples. Furthermore, $\mathbb{E}[\xi_t]=0$, since samples are drawn uniformly at random.
	Observe that
		\begin{align}\label{eqthetas}
			\begin{split}
				&L_{\text{outer1}}({\omega}^*_{t+1}(\theta_{t+1}))-L_{\text{outer1}}({\omega}^*_t(\theta_t)) \\
				=& \left\{L_{\text{outer1}}({\omega}^*_{t+1}(\theta_{t+1}))- L_{\text{outer1}}({\omega}^*_{t+1}(\theta_t))\right\} +\left\{L_{\text{outer1}}({\omega}^*_{t+1}(\theta_t))-L_{\text{outer1}}({\omega}^*_t(\theta_t))\right\}.
			\end{split}
		\end{align}
	
	For the first term, by Lipschitz continuity of $\nabla_{\theta} L_{\text{outer1}}({\omega}^*_{t+1}(\theta))$, we can deduce that:
	\begin{align*}
		&L_{\text{outer1}}({\omega}^*_{t+1}(\theta_{t+1}))-L_{\text{outer1}}({\omega}^*_{t+1}(\theta_t))  \\
		\leq &\left \langle\nabla_{\theta} L_{\text{outer1}}({\omega}^*_{t+1}(\theta_t)),\theta_{t+1}-\theta_t \right\rangle + \frac{L}{2} \left\|\theta_{t+1}-\theta_t\right\|_2^2\\
		 = & \left\langle\nabla_{\theta} L_{\text{outer1}}({\omega}^*_{t+1}(\theta_t)), -\beta_t [\nabla_{\theta} L_{\text{outer1}}({\omega}^*_{t+1}(\theta_t))+\xi_t ] \right\rangle + \frac{L\beta_t^2}{2} \left\|\nabla_{\theta} L_{\text{outer1}}({\omega}^*_{t+1}(\theta_t))+\xi_t\right\|_2^2\\
		 =& -(\beta_t-\frac{L\beta_t^2}{2}) \left\|\nabla_{\theta} L_{\text{outer1}}({\omega}^*_{t+1}(\theta_t))\right\|_2^2
	+ \frac{L\beta_t^2}{2}\|\xi_t\|_2^2 - (\beta_t-L\beta_t^2)\langle \nabla_{\theta} L_{\text{outer1}}({\omega}^*_t(\theta_t)),\xi_t\rangle.
	\end{align*}

For the second term, by Lipschitz smoothness of the meta loss function $L_{\text{outer1}}({\omega}^*_{t+1}(\theta_{t+1}))$,  we have
\begin{align*}
	&L_{\text{outer1}}({\omega}^*_{t+1}(\theta_t))- L_{\text{outer1}}({\omega}^*_t(\theta_t)) \\
	\leq & \left\langle \nabla_{\omega} L_{\text{outer1}}({\omega}^*_t(\theta_t)), {\omega}^*_{t+1}(\theta_t)-{\omega}^*_t(\theta_t) \right\rangle
	+ \frac{L}{2}\|{\omega}^*_{t+1}(\theta_t)-{\omega}^*_t(\theta_t)\|_2^2.
\end{align*}
Since
\begin{align*}
	{\omega}^*_{t+1}(\theta_t)-{\omega}^*_t(\theta_t)
	= - \alpha_t \nabla_{\omega}L_{\text{inner}}(\omega_t;\theta_t)|_{\Psi_t},
\end{align*}
where $\Psi_t$ denotes the mini-batch drawn randomly from the training dataset  in the $t$-th iteration, $\nabla_{\omega}L_{\text{inner}}(\omega_t;\theta_t)=
\frac{1}{n} \sum_{i=1}^n \nabla_{\omega_t} L_{\text{inner}}^i(\omega)\Big|_{\omega_t}$. Since the mini-batch $\Psi_t$ is drawn uniformly at random, we can rewrite the update equation as:
\begin{align*}
	{\omega}^*_{t+1}(\theta_t)={\omega}^*_t(\theta_t)  - \alpha_t [\nabla_{\omega}L_{\text{inner}}(\omega_t;\theta_t)+\psi_t],
\end{align*}
where $\psi_t = \nabla_{\omega}L_{\text{inner}}(\omega_t;\theta_t)|_{\Psi_t}-\nabla_{\omega}L_{\text{inner}}(\omega_t;\theta_t)$.
Note that $\psi_t$ are i.i.d. random variables with finite variance, since $\Psi_t$ are drawn i.i.d. with a finite number of samples, and thus $\mathbb{E}[\psi_t]=0$,  $\mathbb{E}[\|\psi_t\|_2^2]\leq \sigma^2$.
Thus we have
\begin{align*}
	& L_{\text{outer1}}({\omega}^*_{t+1}(\theta_t))- L_{\text{outer1}}({\omega}^*_t(\theta_t))  \\
	\leq &  \left\langle \nabla_{\omega} L_{\text{outer1}}({\omega}^*_t(\theta_t)), - \alpha_t [\nabla_{\omega}L_{\text{inner}}(\omega_t;\theta_t)+\psi_t]\right\rangle
	 + \frac{L}{2}\left\|\alpha_t [\nabla_{\omega}L_{\text{inner}}(\omega_t;\theta_t)+\psi_t]\right\|_2^2                     \\
	= &  \frac{L\alpha_t^2}{2} \left\|\nabla_{\omega}L_{\text{inner}}(\omega_t;\theta_t)\right\|_2^2 -\alpha_t \left\langle  \nabla_{\omega} L_{\text{outer1}}({\omega}^*_t(\theta_t)),   \nabla_{\omega}L_{\text{inner}}(\omega_t;\theta_t)     \right\rangle
	+ \frac{L\alpha_t^2}{2}\left\|\psi_t\right\|_2^2 \\
	& +L\alpha_t^2\left\langle \nabla_{\omega}L_{\text{inner}}(\omega_t;\theta_t), \psi_t\right\rangle -\alpha_t \left\langle  \nabla_{\omega} L_{\text{outer1}}({\omega}^*_t(\theta_t)), \psi_t      \right\rangle \\
	\leq & \frac{L\alpha_t^2 \rho^2}{2} + \alpha_t \rho \left\|\nabla_{\omega}L_{\text{inner}}(\omega_t;\theta_t)\right\|_2  + \frac{L\sigma^2 \alpha_t^2}{2}+L\alpha_t^2\left\langle \nabla_{\omega}L_{\text{inner}}(\omega_t;\theta_t), \psi_t\right\rangle -\alpha_t \left\langle  \nabla_{\omega} L_{\text{outer1}}({\omega}^*_t(\theta_t)), \psi_t\right\rangle.
\end{align*}
The last inequality holds since $\left\langle  \nabla_{\omega} L_{\text{outer1}}({\omega}^*_t(\theta_t)),   \nabla_{\omega}L_{\text{inner}}(\omega_t;\theta_t)     \right\rangle \leq \left\|\nabla_{\omega}L_{\text{outer1}}({\omega}^*_t(\theta_t))\right\|_2\left\|\nabla_{\omega}L_{\text{inner}}(\omega_t;\theta_t)\right\|_2$. Thus Eq.(\ref{eqthetas}) satifies
\begin{align*}
	\begin{split}
		&L_{\text{outer1}}({\omega}^*_{t+1}(\theta_{t+1}))-L_{\text{outer1}}({\omega}^*_t(\theta_t)) \\
		\leq & \frac{L\alpha_t^2 \rho^2}{2} + \alpha_t \rho \left\|\nabla_{\omega}L_{\text{inner}}(\omega_t;\theta_t)\right\|_2+ \frac{L\sigma^2 \alpha_t^2}{2}+L\alpha_t^2\left\langle \nabla_{\omega}L_{\text{inner}}(\omega_t;\theta_t), \psi_t\right\rangle -\alpha_t \left\langle  \nabla_{\omega} L_{\text{outer1}}({\omega}^*_t(\theta_t)), \psi_t\right\rangle \\ &-(\beta_t-\frac{L\beta_t^2}{2}) \left\|\nabla_{\theta} L_{\text{outer1}}({\omega}^*_{t+1}(\theta_t))\right\|_2^2
		+ \frac{L\beta_t^2}{2}\|\xi_t\|_2^2 - (\beta_t-L\beta_t^2)\langle \nabla_{\theta} L_{\text{outer1}}({\omega}^*_t(\theta_t)),\xi_t\rangle.
	\end{split}
\end{align*}
Rearranging the terms, and taking expectations with respect to $\xi_t$ and $\psi_t$ on both sides, we can obtain
\begin{align*}
	\begin{split}
		&(\beta_t-\frac{L\beta_t^2}{2})
		\left\|\nabla_{\theta} L_{\text{outer1}}({\omega}^*_{t+1}(\theta_t))\right\|_2^2 \\
		\leq & \frac{L\alpha_t^2 \rho^2}{2} + \alpha_t \rho \left\|\nabla_{\omega}L_{\text{inner}}(\omega_t;\theta_t)\right\|_2 + \frac{L\sigma^2 \alpha_t^2}{2} +L_{\text{outer1}}({\omega}^*_t(\theta_t))-L_{\text{outer1}}({\omega}^*_{t+1}(\theta_{t+1}))
		+ \frac{L\beta_t^2}{2}\sigma^2,
	\end{split}
\end{align*}
	since $\mathbb{E}[\xi_t]=0,\mathbb{E}[\psi_t]=0$ and $\mathbb{E} [\|\xi_t\|_2^2] \leq \sigma^2$. Summing up the above inequalities, we can obtain
	\begin{align*}\label{eqrand}
		\begin{split}
			&\sum\nolimits_{t=1}^T (\beta_t-\frac{L\beta_t^2}{2})
			\left\|\nabla_{\theta} L_{\text{outer1}}({\omega}^*_{t+1}(\theta_t))\right\|_2^2 \\
			 \leq & L_{\text{outer1}}({\omega}^{*{(1)}}(\theta^{(1)})) -L_{\text{outer1}}({\omega}^{*{(T+1)}}(\theta^{(T+1)})) +
			\frac{L(\sigma^2+\rho^2)}{2} \sum_{t=1}^T\alpha_t^2 +\rho\sum_{t=1}^T\alpha_t \left\|\nabla_{\omega}L_{\text{inner}}(\omega_t;\theta_t)\right\|_2  +\frac{L}{2}\sum_{t=1}^T\beta_t^2\sigma^2 \\
			\leq & L_{\text{outer1}}({\omega}^{*{(1)}}(\theta^{(1)}))  +
			\frac{L(\sigma^2+\rho^2)}{2} \sum_{t=1}^T\alpha_t^2 +\rho\sum_{t=1}^T\alpha_t \left\|\nabla_{\omega}L_{\text{inner}}(\omega_t;\theta_t)\right\|_2  +\frac{L}{2}\sum_{t=1}^T\beta_t^2\sigma^2 .
%			&\leq L_{\text{outer1}}({\omega}^{*(1)})(\theta^{(1)}) +\sum_{t=1}^T\alpha_t\rho^2 (1+\frac{\alpha_t L}{2}) -\sum_{t=1}^T(\beta_t-L\beta_t^2)\langle \nabla_{\theta} L_{\text{outer1}}({\omega}^*_t(\theta_t)),\xi_t\rangle +\frac{L}{2}\sum_{t=1}^T\beta_t^2\|\xi_t\|_2^2,
		\end{split}
	\end{align*}
	%	\begin{align}\label{eqrand}
	%	\begin{split}
	%	& \sum_{t=1}^T (\beta_t-\frac{L\beta_t^2}{2}) \|\nabla L_{\text{outer1}}(\theta_t)\|_2^2 \\
	%	&\leq L_{\text{outer1}}(\theta^{(1)})-L_{\text{outer1}}(\theta^{(N+1)}) - \sum_{t=1}^T (\beta_t-L\beta_t^2)\langle \nabla L_{\text{outer1}}(\theta_t),\xi_t\rangle +\frac{L}{2}\sum_{t=1}^T \beta_t^2\|\xi_t\|_2^2 \\
	%	&\leq L_{\text{outer1}}(\theta^{(1)})-L_{\text{outer1}}(\theta^*) - \sum_{t=1}^T (\beta_t-L\beta_t^2)\langle \nabla L_{\text{outer1}}(\theta_t),\xi_t\rangle +\frac{L}{2}\sum_{t=1}^T \beta_t^2\|\xi_t\|_2^2,
	%	\end{split}
	%	\end{align}
	%	where the last inequality follows from the fact that $L_{\text{outer1}}(\theta_{t+1})\geq L_{\text{outer1}}(\theta^*)$.
%	Taking expectations with respect to $\xi^{(N)}$ on both sides of Eq. \ref{eqrand}, we can then obtain:
%		\begin{align*}
%			\begin{split}
%				& \sum_{t=1}^T (\beta_t-\frac{L\beta_t^2}{2})\mathbb{E}_{\xi^{(N)}} \left\|\nabla_{\theta} L_{\text{outer1}}({\omega}^*_{t+1}(\theta_t))\right\|_2^2
%			 \leq L_{\text{outer1}}({\omega}^{*(1)})(\theta^{(1)})+\sum_{t=1}^T\alpha_t\rho^2 (1+\frac{\alpha_t L}{2})+ \left(\frac{L\sigma^2}{2}+ \right)\sum_{t=1}^T \alpha_t^2 + \frac{L\sigma^2}{2} \sum_{t=1}^T \beta_t^2,
%			\end{split}
%		\end{align*}
	Furthermore, we can deduce that
	\begin{align*}
		\begin{split}
			&\min_{t} \mathbb{E} \left[ \left\|\nabla_{\theta} L_{\text{outer1}}({\omega}^*_{t+1}(\theta_t))\right\|_2^2 \right] \\
			\leq & \frac{\sum_{t=1}^T (\beta_t-\frac{L\beta_t^2}{2})\mathbb{E} \left\|\nabla_{\theta} L_{\text{outer1}}({\omega}^*_{t+1}(\theta_t))\right\|_2^2 }{\sum_{t=1}^{T} \left(\beta_t-\frac{L\beta_t^2}{2}\right)}\\
		\leq 	& \frac{1}{\sum_{t=1}^T (2\beta_t-L\beta_t^2)} \left[L_{\text{outer1}}({\omega}^{*{(1)}}(\theta^{(1)}))+\frac{L(\sigma^2+\rho^2)}{2} \sum_{t=1}^T\alpha_t^2 +\rho\sum_{t=1}^T\alpha_t \left\|\nabla_{\omega}L_{\text{inner}}(\omega_t;\theta_t)\right\|_2  +\frac{L}{2}\sum_{t=1}^T\beta_t^2\sigma^2  \right]\\
			 \leq & \frac{1}{\sum_{t=1}^T \beta_t} \left[L_{\text{outer1}}({\omega}^{*{(1)}}(\theta^{(1)}))+\frac{L(\sigma^2+\rho^2)}{2} \sum_{t=1}^T\alpha_t^2 +\rho\sum_{t=1}^T\alpha_t \left\|\nabla_{\omega}L_{\text{inner}}(\omega_t;\theta_t)\right\|_2  +\frac{L}{2}\sigma^2 \sum_{t=1}^T\beta_t^2 \right] \\
		 \leq	& \frac{1}{T\beta_T} \left[L_{\text{outer1}}({\omega}^{*(1)}(\theta^{(1)}))+\frac{L(\sigma^2+\rho^2)}{2} \sum_{t=1}^T\alpha_t^2 +\rho\sum_{t=1}^T\alpha_t \left\|\nabla_{\omega}L_{\text{inner}}(\omega_t;\theta_t)\right\|_2  +\frac{L}{2}\sigma^2 \sum_{t=1}^T\beta_t^2\right]\\
			%& = \frac{2L_{\text{outer1}}({\omega}^{*(1)})(\theta^{(1)})+L\rho^2 \sum_{t=1}^T \alpha_t^2 + L\sigma^2 \sum_{t=1}^T \beta_t^2 }{T} \frac{1}{\beta_T} + \frac{2\alpha_1 \rho^2}{\beta_T}\\
			 = &\frac{L_{\text{outer1}}({\omega}^{*(1)}(\theta^{(1)}))+\frac{L(\sigma^2+\rho^2)}{2} \sum_{t=1}^T\alpha_t^2 +\rho\sum_{t=1}^T\alpha_t \left\|\nabla_{\omega}L_{\text{inner}}(\omega_t;\theta_t)\right\|_2  +\frac{L}{2}\sigma^2 \sum_{t=1}^T\beta_t^2 }{T} \max\{L,\frac{\sqrt{T}}{c}\}   \\
			 = & \mathcal{O}(\frac{C}{\sqrt{T}}).
		\end{split}
	\end{align*}
	The third inequality holds since $\sum_{t=1}^T (2\beta_t-L\beta_t^2) \!=\! \sum_{t=1}^T \beta_t (2-L\beta_t) \!\geq\! \sum_{t=1}^T \beta_t$, and the final equality holds since $\lim_{T\to \infty} \sum_{t=1}^T\alpha_t^2 $\  $< \infty, \lim_{T\to \infty} \sum_{t=1}^T\beta_t^2 < \infty, \lim_{T\to \infty} \sum_{t=1}^T\alpha_t \left\|\nabla_{\omega}L_{\text{inner}}(\omega_t;\theta_t)\right\|_2 < \infty $.
	Thus we can conclude that our algorithm can always achieve $\min_{0\leq t \leq T} \mathbb{E}[ \|\nabla_{\theta} L_{\text{outer1}}(\theta_t)\|_2^2] \leq \mathcal{O}(\frac{C}{\sqrt{T}})$ in $T$ steps, and this finishes our proof of Theorem \ref{ths1}.
\end{proof}

%\begin{Lemma}\label{lemma2}
%	(Lemma A.5 in \cite{mairal2013stochastic}) Let $(a_n)_{n\leq 1},(b_n)_{n\leq 1}$ be two non-negative real sequences such that the series $\sum_{i=1}^{\infty} a_n$ diverges, the series $\sum_{i=1}^{\infty} a_nb_n$ converges, and there exists $K>0$ such that $|b_{n+1}-b_n|\leq K a_n$. Then the sequences $(b_n)_{n\leq 1}$ converges to 0.\\
%\end{Lemma}
\begin{Theorem} \label{ths2}
Suppose the loss function $L_{\text{outer1}}$ is Lipschitz smooth with constant $L$, and the gradient of $\phi$ with respect to the loss function $L_{\text{outer2}}$ is Lipschitz continuous with constant $L$. Let the learning rate $\alpha_t, \beta_t, 1\leq t\leq T$ be monotonically descent sequences, and satisfy $\alpha_t=\min\{\frac{1}{L},\frac{c_1}{\sqrt{T}}\}, \beta_t=\min\{\frac{1}{L},\frac{c_2}{\sqrt{T}}\}$, for some $c_1,c_2>0$, such that $\frac{\sqrt{T}}{c_1}\geq L, \frac{\sqrt{T}}{c_2}\geq L$. Meanwhile, they satisfy $\sum_{t=1}^\infty \alpha_t = \infty,\sum_{t=1}^\infty \alpha_t^2 < \infty ,\sum_{t=1}^\infty \beta_t = \infty,\sum_{t=1}^\infty \beta_t^2 < \infty $. Then DuMeta++ can achieve $\mathbb{E}[ \|\nabla L_{\text{outer2}}({\omega}^*_t(\phi_t))\|_2^2] \leq \epsilon$ in $\mathcal{O}(1/\epsilon^2)$ steps. More specifically,
	\begin{align}
		\min_{0\leq t \leq T} \mathbb{E}\left[ \left\|\nabla L_{\text{outer2}}({\omega}^*_t(\phi_t))\right\|_2^2\right] \leq \mathcal{O}(\frac{C}{\sqrt{T}}),
	\end{align}
	where $C$ is some constant independent of the convergence process.
\end{Theorem}

\begin{proof}
	The update equation of $\phi$ in each iteration is as follows:
	\begin{align*}
		\phi_{t+1} =  \phi_t -\beta_t \frac{1}{m}\sum_{i=1}^{m} \nabla_{ \phi} L_{\text{outer2}}^i({\omega}^* _{t+1}(\phi))\Big|_{\phi_t}.
	\end{align*}
	
Under the sampled mini-batch $\Xi_t$, the updating equation can be rewritten as:
	\begin{align*}
		\phi_{t+1} =  \phi_t -\beta_t \nabla_{ \phi}L_{\text{outer2}}({\omega}^*_{t+1}(\phi_t))\big|_{\Xi_t}.
	\end{align*}
	Since the mini-batch is drawn uniformly from the entire data set, the above update equation can be written as:
	\begin{align*}
		\phi_{t+1} =  \phi_t -\beta_t[ \nabla_{ \phi}L_{\text{outer2}}({\omega}^*_{t+1}(\phi_t))+\xi_t],
	\end{align*}
	where $\xi_t = \nabla_{ \phi}L_{\text{outer2}}({\omega}^*_{t+1}(\phi_t))\big|_{\Xi_t}-\nabla_{ \phi}L_{\text{outer2}}({\omega}^*_{t+1}(\phi_t))$. Note that $\xi_t$ are i.i.d random variable with finite variance, since $\Xi_t$ are drawn i.i.d with a finite number of samples. Furthermore, $\mathbb{E}[\xi_t]=0$, since samples are drawn uniformly at random.
	Observe that
		\begin{align}\label{eqphis}
			\begin{split}
				&L_{\text{outer2}}({\omega}^*_{t+1}(\phi_{t+1}))-L_{\text{outer2}}({\omega}^*_t(\phi_t)) \\
				=& \left\{L_{\text{outer2}}({\omega}^*_{t+1}(\phi_{t+1}))- L_{\text{outer2}}({\omega}^*_{t+1}(\phi_t))\right\} +\left\{L_{\text{outer2}}({\omega}^*_{t+1}(\phi_t))-L_{\text{outer2}}({\omega}^*_t(\phi_t))\right\}.
			\end{split}
		\end{align}
	
	For the first term, by Lipschitz continuity of $\nabla_{\phi} L_{\text{outer2}}({\omega}^*_{t+1}(\phi))$, we can deduce that:
	\begin{align*}
		&L_{\text{outer2}}({\omega}^*_{t+1}(\phi_{t+1}))-L_{\text{outer2}}({\omega}^*_{t+1}(\phi_t))  \\
		\leq &\left \langle\nabla_{\phi} L_{\text{outer2}}({\omega}^*_{t+1}(\phi_t)),\phi_{t+1}-\phi_t \right\rangle + \frac{L}{2} \left\|\phi_{t+1}-\phi_t\right\|_2^2\\
		 = & \left\langle\nabla_{\phi} L_{\text{outer2}}({\omega}^*_{t+1}(\phi_t)), -\beta_t [\nabla_{\phi} L_{\text{outer2}}({\omega}^*_{t+1}(\phi_t))+\xi_t ] \right\rangle + \frac{L\beta_t^2}{2} \left\|\nabla_{\phi} L_{\text{outer2}}({\omega}^*_{t+1}(\phi_t))+\xi_t\right\|_2^2\\
		 =& -(\beta_t-\frac{L\beta_t^2}{2}) \left\|\nabla_{\phi} L_{\text{outer2}}({\omega}^*_{t+1}(\phi_t))\right\|_2^2
	+ \frac{L\beta_t^2}{2}\|\xi_t\|_2^2 - (\beta_t-L\beta_t^2)\langle \nabla_{\phi} L_{\text{outer2}}({\omega}^*_t(\phi_t)),\xi_t\rangle.
	\end{align*}

For the second term, by Lipschitz smoothness of the meta loss function $L_{\text{outer2}}({\omega}^*_{t+1}(\phi_{t+1}))$,  we have
\begin{align*}
	&L_{\text{outer2}}({\omega}^*_{t+1}(\phi_t))- L_{\text{outer2}}({\omega}^*_t(\phi_t)) \\
	\leq & \left\langle \nabla_{\omega} L_{\text{outer2}}({\omega}^*_t(\phi_t)), {\omega}^*_{t+1}(\phi_t)-{\omega}^*_t(\phi_t) \right\rangle
	+ \frac{L}{2}\|{\omega}^*_{t+1}(\phi_t)-{\omega}^*_t(\phi_t)\|_2^2.
\end{align*}
Since
\begin{align*}
	{\omega}^*_{t+1}(\phi_t)-{\omega}^*_t(\phi_t)
	= - \alpha_t \nabla_{\omega}L_{\text{inner}}(\omega_t;\phi_t)|_{\Psi_t},
\end{align*}
where $\Psi_t$ denotes the mini-batch drawn randomly from the training dataset  in the $t$-th iteration, $\nabla_{\omega}L_{\text{inner}}(\omega_t;\phi_t)=
\frac{1}{n} \sum_{i=1}^n \nabla_{\omega_t} L_{\text{inner}}^i(\omega)\Big|_{\omega_t}$. Since the mini-batch $\Psi_t$ is drawn uniformly at random, we can rewrite the update equation as:
\begin{align*}
	{\omega}^*_{t+1}(\phi_t)={\omega}^*_t(\phi_t)  - \alpha_t [\nabla_{\omega}L_{\text{inner}}(\omega_t;\phi_t)+\psi_t],
\end{align*}
where $\psi_t = \nabla_{\omega}L_{\text{inner}}(\omega_t;\phi_t)|_{\Psi_t}-\nabla_{\omega}L_{\text{inner}}(\omega_t;\phi_t)$.
Note that $\psi_t$ are i.i.d. random variables with finite variance, since $\Psi_t$ are drawn i.i.d. with a finite number of samples, and thus $\mathbb{E}[\psi_t]=0$,  $\mathbb{E}[\|\psi_t\|_2^2]\leq \sigma^2$.
Thus we have
\begin{align*}
	& L_{\text{outer2}}({\omega}^*_{t+1}(\phi_t))- L_{\text{outer2}}({\omega}^*_t(\phi_t))  \\
	\leq &  \left\langle \nabla_{\omega} L_{\text{outer2}}({\omega}^*_t(\phi_t)), - \alpha_t [\nabla_{\omega}L_{\text{inner}}(\omega_t;\phi_t)+\psi_t]\right\rangle
	 + \frac{L}{2}\left\|\alpha_t [\nabla_{\omega}L_{\text{inner}}(\omega_t;\phi_t)+\psi_t]\right\|_2^2                     \\
	= &  \frac{L\alpha_t^2}{2} \left\|\nabla_{\omega}L_{\text{inner}}(\omega_t;\phi_t)\right\|_2^2 -\alpha_t \left\langle  \nabla_{\omega} L_{\text{outer2}}({\omega}^*_t(\phi_t)),   \nabla_{\omega}L_{\text{inner}}(\omega_t;\phi_t)     \right\rangle
	+ \frac{L\alpha_t^2}{2}\left\|\psi_t\right\|_2^2 \\
	& +L\alpha_t^2\left\langle \nabla_{\omega}L_{\text{inner}}(\omega_t;\phi_t), \psi_t\right\rangle -\alpha_t \left\langle  \nabla_{\omega} L_{\text{outer2}}({\omega}^*_t(\phi_t)), \psi_t      \right\rangle \\
	\leq & \frac{L\alpha_t^2 \rho^2}{2} + \alpha_t \rho \left\|\nabla_{\omega}L_{\text{inner}}(\omega_t;\phi_t)\right\|_2  + \frac{L\sigma^2 \alpha_t^2}{2}+L\alpha_t^2\left\langle \nabla_{\omega}L_{\text{inner}}(\omega_t;\phi_t), \psi_t\right\rangle -\alpha_t \left\langle  \nabla_{\omega} L_{\text{outer2}}({\omega}^*_t(\phi_t)), \psi_t\right\rangle.
\end{align*}
The last inequality holds since $\left\langle  \nabla_{\omega} L_{\text{outer2}}({\omega}^*_t(\phi_t)),   \nabla_{\omega}L_{\text{inner}}(\omega_t;\phi_t)     \right\rangle \leq \left\|\nabla_{\omega}L_{\text{outer2}}({\omega}^*_t(\phi_t))\right\|_2\left\|\nabla_{\omega}L_{\text{inner}}(\omega_t;\phi_t)\right\|_2$. Thus Eq.(\ref{eqphis}) satifies
\begin{align*}
	\begin{split}
		&L_{\text{outer2}}({\omega}^*_{t+1}(\phi_{t+1}))-L_{\text{outer2}}({\omega}^*_t(\phi_t)) \\
		\leq & \frac{L\alpha_t^2 \rho^2}{2} + \alpha_t \rho \left\|\nabla_{\omega}L_{\text{inner}}(\omega_t;\phi_t)\right\|_2+ \frac{L\sigma^2 \alpha_t^2}{2}+L\alpha_t^2\left\langle \nabla_{\omega}L_{\text{inner}}(\omega_t;\phi_t), \psi_t\right\rangle -\alpha_t \left\langle  \nabla_{\omega} L_{\text{outer2}}({\omega}^*_t(\phi_t)), \psi_t\right\rangle \\ &-(\beta_t-\frac{L\beta_t^2}{2}) \left\|\nabla_{\phi} L_{\text{outer2}}({\omega}^*_{t+1}(\phi_t))\right\|_2^2
		+ \frac{L\beta_t^2}{2}\|\xi_t\|_2^2 - (\beta_t-L\beta_t^2)\langle \nabla_{\phi} L_{\text{outer2}}({\omega}^*_t(\phi_t)),\xi_t\rangle.
	\end{split}
\end{align*}
Rearranging the terms, and taking expectations with respect to $\xi_t$ and $\psi_t$ on both sides, we can obtain
\begin{align*}
	\begin{split}
		&(\beta_t-\frac{L\beta_t^2}{2})
		\left\|\nabla_{\phi} L_{\text{outer2}}({\omega}^*_{t+1}(\phi_t))\right\|_2^2 \\
		\leq & \frac{L\alpha_t^2 \rho^2}{2} + \alpha_t \rho \left\|\nabla_{\omega}L_{\text{inner}}(\omega_t;\phi_t)\right\|_2 + \frac{L\sigma^2 \alpha_t^2}{2} +L_{\text{outer2}}({\omega}^*_t(\phi_t))-L_{\text{outer2}}({\omega}^*_{t+1}(\phi_{t+1}))
		+ \frac{L\beta_t^2}{2}\sigma^2,
	\end{split}
\end{align*}
	since $\mathbb{E}[\xi_t]=0,\mathbb{E}[\psi_t]=0$ and $\mathbb{E} [\|\xi_t\|_2^2] \leq \sigma^2$. Summing up the above inequalities, we can obtain
	\begin{align*}\label{eqrand}
		\begin{split}
			&\sum\nolimits_{t=1}^T (\beta_t-\frac{L\beta_t^2}{2})
			\left\|\nabla_{\phi} L_{\text{outer2}}({\omega}^*_{t+1}(\phi_t))\right\|_2^2 \\
			 \leq & L_{\text{outer2}}({\omega}^{*(1)}(\phi^{(1)})) -L_{\text{outer2}}({\omega}^{*(T+1)}(\phi^{(T+1)})) +
			\frac{L(\sigma^2+\rho^2)}{2} \sum_{t=1}^T\alpha_t^2 +\rho\sum_{t=1}^T\alpha_t \left\|\nabla_{\omega}L_{\text{inner}}(\omega_t;\phi_t)\right\|_2  +\frac{L}{2}\sum_{t=1}^T\beta_t^2\sigma^2 \\
			\leq & L_{\text{outer2}}({\omega}^{*(1)}(\phi^{(1)}))  +
			\frac{L(\sigma^2+\rho^2)}{2} \sum_{t=1}^T\alpha_t^2 +\rho\sum_{t=1}^T\alpha_t \left\|\nabla_{\omega}L_{\text{inner}}(\omega_t;\phi_t)\right\|_2  +\frac{L}{2}\sum_{t=1}^T\beta_t^2\sigma^2 .
%			&\leq L_{\text{outer2}}({\omega}^{*(1)})(\phi^{(1)}) +\sum_{t=1}^T\alpha_t\rho^2 (1+\frac{\alpha_t L}{2}) -\sum_{t=1}^T(\beta_t-L\beta_t^2)\langle \nabla_{\phi} L_{\text{outer2}}({\omega}^*_t(\phi_t)),\xi_t\rangle +\frac{L}{2}\sum_{t=1}^T\beta_t^2\|\xi_t\|_2^2,
		\end{split}
	\end{align*}
	%	\begin{align}\label{eqrand}
	%	\begin{split}
	%	& \sum_{t=1}^T (\beta_t-\frac{L\beta_t^2}{2}) \|\nabla L_{\text{outer2}}(\phi_t)\|_2^2 \\
	%	&\leq L_{\text{outer2}}(\phi^{(1)})-L_{\text{outer2}}(\phi^{(N+1)}) - \sum_{t=1}^T (\beta_t-L\beta_t^2)\langle \nabla L_{\text{outer2}}(\phi_t),\xi_t\rangle +\frac{L}{2}\sum_{t=1}^T \beta_t^2\|\xi_t\|_2^2 \\
	%	&\leq L_{\text{outer2}}(\phi^{(1)})-L_{\text{outer2}}(\phi^*) - \sum_{t=1}^T (\beta_t-L\beta_t^2)\langle \nabla L_{\text{outer2}}(\phi_t),\xi_t\rangle +\frac{L}{2}\sum_{t=1}^T \beta_t^2\|\xi_t\|_2^2,
	%	\end{split}
	%	\end{align}
	%	where the last inequality follows from the fact that $L_{\text{outer2}}(\phi_{t+1})\geq L_{\text{outer2}}(\phi^*)$.
%	Taking expectations with respect to $\xi^{(N)}$ on both sides of Eq. \ref{eqrand}, we can then obtain:
%		\begin{align*}
%			\begin{split}
%				& \sum_{t=1}^T (\beta_t-\frac{L\beta_t^2}{2})\mathbb{E}_{\xi^{(N)}} \left\|\nabla_{\phi} L_{\text{outer2}}({\omega}^*_{t+1}(\phi_t))\right\|_2^2
%			 \leq L_{\text{outer2}}({\omega}^{*(1)})(\phi^{(1)})+\sum_{t=1}^T\alpha_t\rho^2 (1+\frac{\alpha_t L}{2})+ \left(\frac{L\sigma^2}{2}+ \right)\sum_{t=1}^T \alpha_t^2 + \frac{L\sigma^2}{2} \sum_{t=1}^T \beta_t^2,
%			\end{split}
%		\end{align*}
	Furthermore, we can deduce that
	\begin{align*}
		\begin{split}
			&\min_{t} \mathbb{E} \left[ \left\|\nabla_{\phi} L_{\text{outer2}}({\omega}^*_{t+1}(\phi_t))\right\|_2^2 \right] \\
			\leq & \frac{\sum_{t=1}^T (\beta_t-\frac{L\beta_t^2}{2})\mathbb{E} \left\|\nabla_{\phi} L_{\text{outer2}}({\omega}^*_{t+1}(\phi_t))\right\|_2^2 }{\sum_{t=1}^{T} \left(\beta_t-\frac{L\beta_t^2}{2}\right)}\\
		\leq 	& \frac{1}{\sum_{t=1}^T (2\beta_t-L\beta_t^2)} \left[L_{\text{outer2}}({\omega}^{*(1)}(\phi^{(1)}))+\frac{L(\sigma^2+\rho^2)}{2} \sum_{t=1}^T\alpha_t^2 +\rho\sum_{t=1}^T\alpha_t \left\|\nabla_{\omega}L_{\text{inner}}(\omega_t;\phi_t)\right\|_2  +\frac{L}{2}\sum_{t=1}^T\beta_t^2\sigma^2  \right]\\
			 \leq & \frac{1}{\sum_{t=1}^T \beta_t} \left[L_{\text{outer2}}({\omega}^{*(1)}(\phi^{(1)}))+\frac{L(\sigma^2+\rho^2)}{2} \sum_{t=1}^T\alpha_t^2 +\rho\sum_{t=1}^T\alpha_t \left\|\nabla_{\omega}L_{\text{inner}}(\omega_t;\phi_t)\right\|_2  +\frac{L}{2}\sigma^2 \sum_{t=1}^T\beta_t^2 \right] \\
		 \leq	& \frac{1}{T\beta_T} \left[L_{\text{outer2}}({\omega}^{*(1)}(\phi^{(1)}))+\frac{L(\sigma^2+\rho^2)}{2} \sum_{t=1}^T\alpha_t^2 +\rho\sum_{t=1}^T\alpha_t \left\|\nabla_{\omega}L_{\text{inner}}(\omega_t;\phi_t)\right\|_2  +\frac{L}{2}\sigma^2 \sum_{t=1}^T\beta_t^2\right]\\
			%& = \frac{2L_{\text{outer2}}({\omega}^{*(1)})(\phi^{(1)})+L\rho^2 \sum_{t=1}^T \alpha_t^2 + L\sigma^2 \sum_{t=1}^T \beta_t^2 }{T} \frac{1}{\beta_T} + \frac{2\alpha_1 \rho^2}{\beta_T}\\
			 = &\frac{L_{\text{outer2}}({\omega}^{*(1)}(\phi^{(1)}))+\frac{L(\sigma^2+\rho^2)}{2} \sum_{t=1}^T\alpha_t^2 +\rho\sum_{t=1}^T\alpha_t \left\|\nabla_{\omega}L_{\text{inner}}(\omega_t;\phi_t)\right\|_2  +\frac{L}{2}\sigma^2 \sum_{t=1}^T\beta_t^2 }{T} \max\{L,\frac{\sqrt{T}}{c}\}   \\
			 = & \mathcal{O}(\frac{C}{\sqrt{T}}).
		\end{split}
	\end{align*}
	The third inequality holds since $\sum_{t=1}^T (2\beta_t-L\beta_t^2) \!=\! \sum_{t=1}^T \beta_t (2-L\beta_t) \!\geq\! \sum_{t=1}^T \beta_t$, and the final equality holds since $\lim_{T\to \infty} \sum_{t=1}^T\alpha_t^2 $\  $< \infty, \lim_{T\to \infty} \sum_{t=1}^T\beta_t^2 < \infty, \lim_{T\to \infty} \sum_{t=1}^T\alpha_t \left\|\nabla_{\omega}L_{\text{inner}}(\omega_t;\phi_t)\right\|_2 < \infty $.
	Thus we can conclude that our algorithm can always achieve $\min_{0\leq t \leq T} \mathbb{E}[ \|\nabla_{\phi} L_{\text{outer2}}(\phi_t)\|_2^2] \leq \mathcal{O}(\frac{C}{\sqrt{T}})$ in $T$ steps, and this finishes our proof of Theorem \ref{ths2}.
\end{proof}

\section{Additional Implementation Details}
\subsection{Architectures}
Table \ref{tab:U-Net} provides the U-Net architecture utilized for both pretraining and finetuning.
Four similar layers as layer 27 are attached following
layers 14, 17, 20, 23 for deep supervision.

\subsection{Additional Training Details}
Layers 0 to 11 are the encoder and layers 12 to 27 are the decoder. As for  different base/meta-learner splits in Table 4 in the manuccript, ``ft 5 umsample layers" means fine-tuning layers 12 to 27; ``ft 5 umsample layers" means fine-tuning layers 12 to 27; ``ft 4 umsample layers" means fine-tuning layers 15 to 27; ``ft 3 umsample layers" means fine-tuning layers 18 to 27; ``ft 2 umsample layers" means fine-tuning layers 21 to 27; ``ft 1 umsample layers" means fine-tuning layers 24 to 27.
Our class-aware regularizations are applied in layers \{1, 3, 5, 7, 9, 14, 17, 20, 23, 26\}.

% \usepackage{booktabs}
\begin{table}[h]
\centering
\caption{U-Net architecture. IN: Instance Normalization. The batch size dimension is denoted as b. c (32) is the starting
channel width multiplier of the model. n (4) indicates the number of output channels and is set to the number of labels. The up- and down- sampling layers are in bold. We omit deep supervision layers.}
\label{tab:U-Net}
\begin{tabular}{@{}lll@{}}
\toprule
id & Layer                                                   & Output size          \\ \midrule
0  & Conv3D(1,c,3,1,1), IN, ReLU                             & b,c,h,w,d            \\
1  & Conv3D(c,c,3,1,1), IN, ReLU                             & b,c,h,w,d            \\
2  & \textbf{Conv3D(c,2c,3,2,1), IN, ReLU}                   & b,2c,h/2,w/2,d/2     \\
3  & Conv3D(2c,2c,3,1,1), IN, ReLU                           & b,2c,h/2,w/2,d/2     \\
4  & \textbf{Conv3D(2c,4c,3,2,1), IN, ReLU}                  & b,4c,h/4,w/4,d/4     \\
5  & Conv3D(4c,4c,3,1,1), IN, ReLU                           & b,4c,h/4,w/4,d/4     \\
6  & \textbf{Conv3D(4c,8c,3,2,1), IN, ReLU}                  & b,8c,h/8,w/8,d/8     \\
7  & Conv3D(8c,8c,3,1,1), IN, ReLU                           & b,8c,h/8,w/8,d/8     \\
8  & \textbf{Conv3D(8c,10c,3,2,1), IN, ReLU}                 & b,10c,h/16,w/16,d/16 \\
9  & Conv3D(10c,10c,3,1,1), IN, ReLU                         & b,10c,h/16,w/16,d/16 \\
10 & \textbf{Conv3D(10c,10c,3,2,1), IN, ReLU}                & b,10c,h/32,w/32,d/32 \\
11 & Conv3D(10c,10c,3,1,1), IN, ReLU                         & b,10c,h/32,w/32,d/32 \\
12 & \textbf{TransConv3D(10c,10c,2,2,0), Concatenate with 9} & b,20c,h/16,w/16,d/16 \\
13 & Conv3D(20c,10c,3,1,1), IN, ReLU                         & b,10c,h/16,w/16,d/16 \\
14 & Conv3D(10c,10c,3,1,1), IN, ReLU                         & b,10c,h/16,w/16,d/16 \\
15 & \textbf{TransConv3D(10c,8c,2,2,0), Concatenate with 7}  & b,16c,h/8,w/8,d/8    \\
16 & Conv3D(16c,8c,3,1,1), IN, ReLU                          & b,8c,h/8,w/8,d/8     \\
17 & Conv3D(8c,8c,3,1,1), IN, ReLU                           & b,8c,h/8,w/8,d/8     \\
18 & \textbf{TransConv3D(8c,4c,2,2,0), Concatenate with 5}   & b,8c,h/4,w/4,d/4     \\
19 & Conv3D(8c,4c,3,1,1), IN, ReLU                           & b,4c,h/4,w/4,d/4     \\
20 & Conv3D(4c,4c,3,1,1), IN, ReLU                           & b,4c,h/4,w/4,d/4     \\
21 & \textbf{TransConv3D(4c,2c,2,2,0), Concatenate with 3}   & b,4c,h/2,w/2,d/2     \\
22 & Conv3D(4c,2c,3,1,1), IN, ReLU                           & b,2c,h/2,w/2,d/2     \\
23 & Conv3D(2c,2c,3,1,1), IN, ReLU                           & b,2c,h/2,w/2,d/2     \\
24 & \textbf{TransConv3D(2c,c,2,2,0), Concatenate with 1}    & b,2c,h,w,d           \\
25 & Conv3D(2c,c,3,1,1), IN, ReLU                            & b,c,h,w,d            \\
26 & Conv3D(c,c,3,1,1), IN, ReLU                             & b,c,h,w,d            \\
27 & Conv3D(c,n,1,1,0), IN, ReLU, Softmax                    & b,n,h,w,d            \\ \bottomrule
\end{tabular}
\end{table}

\end{document}